\documentclass{article}

\usepackage{microtype}
\usepackage{graphicx}
\usepackage{booktabs} 
\usepackage{sidecap}
\usepackage{svg}
\usepackage{tikz}

\usepackage{subcaption}
\newcommand{\foo}[1]{%
\begin{tikzpicture}[#1]%
\draw (0,0) -- (0ex,1ex);%
\draw (0ex,1ex) -- (1ex,0ex);%
\draw (1ex, 0ex) -- (0ex, 0ex);
\end{tikzpicture}%
}

\usepackage{wrapfig}

\usepackage{etoolbox}
\BeforeBeginEnvironment{wrapfigure}{\setlength{\intextsep}{2pt}}
\BeforeBeginEnvironment{wrapfigure}{\setlength{\columnsep}{4pt}}
\usepackage[colorlinks]{hyperref}
\usepackage[accepted]{icml2024}
\usepackage{comment}

\usepackage{icml2024}
\usepackage{styling}

\usepackage{amsmath}
\usepackage{amssymb}
\usepackage{mathtools}
\usepackage{amsthm}

\allowdisplaybreaks
\theoremstyle{plain}
\newtheorem{theorem}{Theorem}[section]

\theoremstyle{definition}
\newtheorem{definition}[theorem]{Definition}

\theoremstyle{remark}

\newcommand{\tbold}[1]{\textbf{#1}}
\usepackage[textsize=tiny]{todonotes}

\icmltitlerunning{PolySketchFormer: Fast Transformers via Sketches for Polynomial Kernels}

\newcommand{\lt}{\textsf{lt}_{\,\foo{}}}

\newcommand{\Attn}{\textsf{Attn}}

\newcommand{\bQ}{\mathbf{Q}}
\newcommand{\bK}{\mathbf{K}}

\begin{document}

\twocolumn[
\icmltitle{PolySketchFormer: Fast Transformers via Sketching Polynomial Kernels}



\icmlsetsymbol{equal}{*}

\begin{icmlauthorlist}
\icmlauthor{Praneeth Kacham}{equal,cmu}
\icmlauthor{Vahab Mirrokni}{equal,comp}
\icmlauthor{Peilin Zhong}{equal,comp}
\end{icmlauthorlist}

\icmlaffiliation{cmu}{Carnegie Mellon University}
\icmlaffiliation{comp}{Google Research}

\icmlcorrespondingauthor{Praneeth Kacham}{pkacham@cs.cmu.edu}
\icmlcorrespondingauthor{Vahab Mirrokni}{mirrokni@google.com}
\icmlcorrespondingauthor{Peilin Zhong}{peilinz@google.com}

\icmlkeywords{Linear Time Attention, Polynomial Kernels, Causal Masking}

\vskip 0.3in
]



\printAffiliationsAndNotice{}  

\begin{abstract}
The quadratic time and memory complexity inherent to self-attention mechanisms, with respect to sequence length, presents a critical computational bottleneck in the training and deployment of large-scale Transformer-based language models. Recent theoretical results indicate the intractability of sub-quadratic softmax attention approximation under reasonable complexity assumptions. This paper addresses this challenge by first demonstrating that polynomial attention with high degree can effectively replace softmax without sacrificing model quality. Next, we develop polynomial sketching techniques from numerical linear algebra to achieve linear-time polynomial attention with approximation guarantees. Crucially, our approach achieves this speedup without requiring the sparsification of attention matrices. We also present a block-based algorithm to apply causal masking efficiently. Combining these techniques, we provide \emph{PolySketchFormer}, a practical linear-time Transformer architecture for language modeling that offers provable guarantees.

We validate PolySketchFormer empirically by training language models capable of handling long contexts. These experiments utilize both synthetic and real-world datasets (PG19, Wikipedia and C4) on Google Cloud TPUs. For context lengths of 32k and GPT-2 style models, our model achieves 2x speedup in training compared to FlashAttention of the fastest configuration, with no observed degradation in quality across our  experiments.\footnote{Our implementation is available at \url{{https://github.com/google-research/google-research/tree/master/polysketchformer}}}
\end{abstract}

\section{Introduction}
Transformer-based models~\citep{vaswani2017attention} are state-of-the-art for many natural language tasks, leading to breakthroughs in machine translation, language understanding \citep{devlin-etal-2019-bert}, and language modeling \citep{brown2020language, chowdhery2022palm, openai2023gpt4, anil2023palm}. However, the quadratic time and space complexity of the attention mechanism limits scalability for long context lengths. Numerous ``efficient transformers'' have been proposed to address this issue \cite{wang2020linformer, katharopoulos2020transformers, choromanski2020rethinking, han2023hyperattention}. These variants approximate\footnote{``Approximation'' is used informally here, since some ``efficient transformers'' deviate significantly from the vanilla model.} the standard attention mechanism. A survey by \citet{tay2020efficient} provides a broad overview of these techniques. While many efficient transformer constructions achieve linear theoretical training complexity, the survey observes that practical training speedups are often less significant, with potential losses in model quality. This explains the continued dominance of vanilla transformers.

In this work, we focus on improving training latency for transformer models in decoding-only tasks, specifically language modeling trained via next-word prediction. Our techniques are generalizable to encoding-only and encoder-decoder transformers, with potential applications beyond language modeling. We will first briefly discuss existing approaches to make training of transformer models faster and then place our contributions in context.

\begin{figure}
    \includegraphics[width=\linewidth]{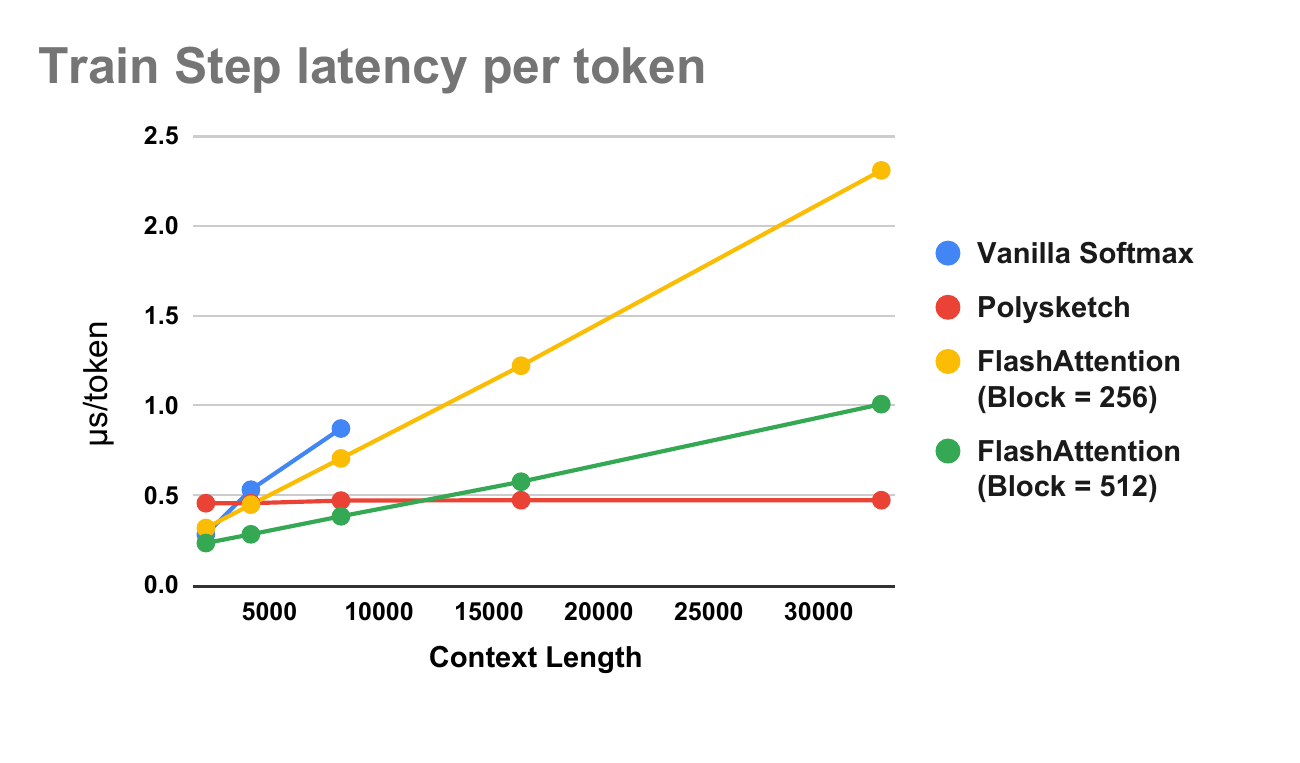}
    \vspace{-0.5in}
    \caption{Train step latency per token in \textmu s/token of GPT-2 small style models with softmax attention (FlashAttention) v.s. ours. Each model is trained with 1M tokens batches. Vanilla softmax attention goes out-of-memory (OOM) for context lengths > 8k.}
    \label{fig:attention_per_token}
\end{figure}
\textbf{Memory efficient and I/O aware approach.}
Recent work by \citep{dao2022flashattention, dao2023flashattention} on FlashAttention and FlashAttention-2 seeks to enable vanilla transformer training on long contexts. This is achieved through I/O-aware optimizations like blocking/tiling and rematerialization, significantly improving memory efficiency. While this reduces the $O(n^2)$\footnote{$n$ denotes the context length -- the number of input tokens.} HBM (High-Bandwidth Memory) requirements of accelerators (GPUs/TPUs), enabling training on thousands of tokens, the computational cost per step remains $O(n^2$) (see Figure~\ref{fig:attention_per_token}), and this remains a barrier to further scaling the context length.

\textbf{Approximate softmax attention via sparsification.}
Another line of work tries to approximate softmax attention and avoid $n\times n$ attention matrix computation by focusing on a smaller set of pairs of \emph{query} and \emph{key} vectors. Techniques include utilizing locality/positional information \cite{child2019generating, beltagy2020longformer,xiao2023efficient,zaheer2020big,roy2021efficient,ding2023longnet}, hashing/bucketing~\cite{kitaev2020reformer,sun2021sparse,han2023hyperattention}, low-rank projection~\cite{wang2020linformer}, or other sparsification methods. In these cases, there is usually some trade-off between model quality and sparsity, i.e., denser attentions improve quality but decrease speed. Hence, an efficient high-quality $n\times n$ attention mechanism may potentially improve on these sparsification-based techniques.

\textbf{Efficient $n\times n$ attention by kernel-based methods.}
The kernel-based view of attention was taken by a series of earlier works  \citep{tsai2019transformer, katharopoulos2020transformers,choromanski2020rethinking,peng2021random}. 
In particular, let $\{\bq_i\in\mathbb{R}^h\}_{i\in[n]}, \{\bk_i\in\mathbb{R}^h\}_{i\in[n]},\{\bv_i\in\mathbb{R}^h\}_{i\in[n]}$ be sets of \emph{query}, \emph{key} and \emph{value} vectors respectively, the output of the attention mechanism for query $\bq_i$ is computed as 
$
\Attn(\bq_i, \{\bk_j\}, \{\bv_j\}) = \sum_{j\in [n]} \frac{\sigma(\bq_i, \bk_j)}{\sum_{j'\in [n]}\sigma(\bq_i, \bk_{j'})} \cdot \bv_j^\top.
$
When the similarity kernel function $\sigma(\bx,\by) := \exp(\langle \bx, \by \rangle)$, the above attention is exactly the softmax attention\footnote{In standard softmax attention, $\sigma(\bx,\by) := \exp(\langle \bx, \by \rangle / \sqrt{h})$. We omit $\sqrt{h}$ here for simplicity of the presentation.}.
If there exists a feature map $\phi$ such that $\sigma(\bx,\by)\equiv \langle \phi(\bx),\phi(\by) \rangle$, 
 the attention output for query $\bq_i$ can be rewritten as:
\begin{align*}
\Attn(\bq_i, \{\bk_j\}, \{\bv_j\}) &= \sum_{j\in [n]} \frac{\phi(\bq_i)^\top \cdot\phi(\bk_j)}{\sum_{j'\in [n]}\phi(\bq_i)^\top \cdot \phi(\bk_{j'})} \cdot \bv_j^\top\\
&= \frac{\phi(\bq_i)^\top \cdot \sum_{j\in [n]}\phi(\bk_j)\cdot \bv_j^\top}{\phi(\bq_i)^\top \cdot \sum_{j'\in [n]}\phi(\bk_{j'})}.
\end{align*}
If $\phi(\cdot)$ has a finite dimension $h'$, one can first compute $ \sum_{j'\in [n]}\phi(\bk_{j'})$ and $\sum_{j\in [n]}\phi(\bk_j)\cdot \bv_j^\top$ in $O(nhh')$ time, and then compute $\Attn(\bq_i, \{\bk_j\}, \{\bv_j\})$ for all $i\in[n]$ in another $O(nhh')$ time, which is linear in the context length $n$.

Most of existing works such as~\cite{katharopoulos2020transformers,bolya2022hydra,tsai2019transformer,babiloni2023linear, yang2023gated,kasai2021finetuning} only considers similarity functions $\sigma(\bx,\by)$ with low dimensional feature mapping (e.g., $\sigma(\bx,\by)=\langle\bx, \by\rangle$, $\langle \bx, \by\rangle^2$, $\langle \mathrm{elu}(\bx)+\mathbf{1}, \mathrm{elu}(\by)+\mathbf{1}\rangle$, etc.).
\citet{hua2022transformer} proposed to use a mixed strategy based on the positions of the tokens: If positions $i,j\in[n]$ are close enough, they use $\sigma(\bq_i,\bk_j)=\mathrm{relu}^2(\langle \bq_i, \bk_j \rangle)$. Otherwise, they use $\sigma(\bq_i,\bk_j)=\langle \bq_i, \bk_j \rangle$, which again has a low dimensional feature mapping.
These simple similarity kernel functions $\sigma(\cdot)$ either suffer from some loss of model quality~\cite{katharopoulos2020transformers} or require additional tweaks of network structures (e.g., significantly increasing the number of attention layers~\cite{hua2022transformer}, introducing decay factors for earlier tokens~\cite{yang2023gated}) to achieve comparable model quality as softmax attention.

Some other previous works try to approximate the regular softmax attention via approximate feature mappings for the exponential similarity function.
Random Feature Attention~\cite{peng2021random} uses random Fourier features to produce an approximate feature mapping but without provable approximation guarantees.
Performer \citep{choromanski2020rethinking} provides a low dimensional approximate non-negative feature mapping $\phi'(\cdot)$ via positive orthogonal random features.
It has provable approximation to the pairwise similarities, i.e., the maximum error
$
\max_{i,j\in[n]}\left|\langle \phi'(\bq_i), \phi'(\bk_j) \rangle - \exp(\langle \bq_i, \bk_j\rangle)\right|
$
is small.
However, the dimension of their feature mapping has to grow exponentially in $\|\bq_i\|_2^2$ and $\|\bk_j\|_2^2$ to have a small error. In other words, consider a single query $\bq_i$ and two keys $\bk_j$ and $\bk_{j'}$ such that all $\|\bq_i\|_2, \|\bk_j\|_2, \|\bk_{j'}\|_2 \le R$, then $\exp(\langle\bq_i, \bk_{j}\rangle)/\exp(\langle \bq_i, \bk_{j'}\rangle) \le \exp(2R^2)$. Thus, the maximum relative probability masses that can be assigned while guaranteeing the approximation factor is limited by the dimension of the feature mapping used. 
In fact, a recent work~\cite{alman2023fast} implies that it is actually impossible to get above approximation for pairwise exponential similarity under Strong Exponential Time Hypothesis (SETH~\cite{impagliazzo2001problems}) when the query and key vectors have large entries. Furthermore, it was observed empirically~\cite{choromanski2020rethinking,hua2022transformer} (also see Figure~\ref{fig:context-length-vs-ppl}) that there is a clear model quality drop in comparison with the exact softmax attention.

Given barriers above, a natural question arises: \emph{Does there exist a similarity kernel function that achieves similar model quality as softmax attention while also admitting proper approximation by a low-dimensional feature mapping?}

\subsection{Our Contributions}\label{sec:our_contribution}
\paragraph{Polynomial similarity kernel function of high degree.} 
To tackle the first part of the above question, we explore the power of the polynomial kernel function $\sigma(\bx,\by)=\langle \bx,\by\rangle^p$ for large even degrees $p\geq 4$ empirically for language modelling tasks.
In particular, we look at the standard GPT-2~\cite{radford2019language} architecture (from the small size to the large size) and the strongest known Transformer recipe (a.k.a. Transformer++) which is a common baseline model studied in many previous works as well~\cite{hua2022transformer,gu2023mamba,yang2023gated}.
We compare the models with vanilla softmax attention to the models that simply replace the attention mechanism with degree-$p$ polynomial attention.
We consider context lengths ranging from 512 to 32k.
As shown in Figure~\ref{fig:context-length-vs-ppl} and our other empirical studies (see Section~\ref{sec:exp} and Appendix), for all synthetic tasks (including tasks for measuring content aware reasoning and memorization capabilities, see Appendix~\ref{apx:synthetic}), autoregressive pre-training metrics (perplexity) and few-shot evaluations that we studied, the models with degree-$p$ polynomial attention ($p\geq 4$) achieve comparable performance as the models with the vanilla softmax attention. 
In addition, we discuss the behavioral similarities between softmax attention and polynomial attention in Section ~\ref{sec:softmax_vs_poly} to provide more intuitions why they had similar empirical outcomes.

\paragraph{Approximate feature mapping for polynomial kernel.}
Unlike exponential kernel whose exact feature mapping has infinite dimension, the feature mapping of degree-$p$ polynomial kernel over $\mathbb{R}^h$ has a finite feature mapping of dimension $h^p$ (see e.g.,~\cite{avron2014subspace}).
In practice, the head size $h$ is usually $64$, $128$ or even $256$~\cite{chowdhery2023palm}.
Therefore, computing the exact feature mapping for $p\geq 4$ is still expensive.
To address this issue, we apply the sketching technique from the numerical linear algebra literature to compute a low-dimensional approximate feature mapping $\phi'$ such that $\langle \phi'(\bx), \phi'(\by) \rangle\approx\langle \bx, \by \rangle^p$. Sketching polynomial kernels~\cite{avron2014subspace,ahle2020oblivious,song2021fast,meister2019tight} has been extensively studied in the literature, and the techniques are used in many applications such as kernel regression~\cite{song2021fast}, kernel PCA~\cite{avron2014subspace}, evaluating element-wise matrix functions~\cite{han2020polynomial}, and etc.
However, though $\langle \bx, \by \rangle^p$ is guaranteed to be non-negative for even integer $p$, none of the approximate feature mappings provided by previous work guarantees $\langle \phi'(\bx) ,\phi'(\by)\rangle \geq 0$.
This is undesired in practice since the original normalized attention weights $\frac{\langle\bq_i,\bk_1\rangle^p}{\sum_{j\in [n]} \langle\bq_i,\bk_j\rangle^p}, \frac{\langle\bq_i,\bk_2\rangle^p}{\sum_{j\in [n]} \langle\bq_i,\bk_j\rangle^p}, \cdots, \frac{\langle\bq_i,\bk_n\rangle^p}{\sum_{j\in [n]} \langle\bq_i,\bk_j\rangle^p}$ naturally represent a probability distribution, but  the property does not hold when there exists some negative attention weight $\langle \phi'(\bq_i), \phi'(\bk_j) \rangle$.
More importantly, previous work~\cite{choromanski2020rethinking,katharopoulos2020transformers} found that negative attention weights make the training process unstable, potentially causing non-convergence.
To address this issue, we open the construction of~\cite{ahle2020oblivious} and develop an approximate feature mapping with desired non-negativity property.
\begin{theorem}\label{thm:main_in_intro}
Let $p\geq 2$ be an even integer, $\varepsilon\in (0,0.5)$ be an error parameter. Let $h$ be the dimension of the vectors to be mapped. There is a randomized feature mapping $\phi' : \R^h \rightarrow \R^{r^2}$ for $r = \Theta(p\varepsilon^{-2}\log 1/\delta)$, such that given any set of vectors $\{\bq_i \in \R^h\}_{i \in [n]}, \{\bk_j \in \R^h\}_{i \in [n]}$:
\begin{enumerate}[nolistsep]
\item $\forall i, j\in[n], \langle \phi'(\bq_i), \phi'(\bk_j) \rangle \geq 0$.
\item 
$
\underset{i,j}{\sum} | \langle \phi'(\bq_i), \phi'(\bk_j)\rangle - \langle \bq_i, \bk_j \rangle^p |^2 \leq \varepsilon^2 \underset{i,j}{\sum} \|\bq_i\|_2^{2p}\|\bk_j\|_2^{2p}
$ holds with probability $1-\delta$.
\item Computing $\phi'(\bx)$ only requires $p/2$ matrix-vector multiplications of matrix size $h\times r$, $(p/2-2)$ matrix-vector multiplications of matrix size $r \times r$, $(p/2-1)$ Hadamard products of $r$-dimensional vectors, and $1$  self-Kronecker product of an $r$-dimensional vector.
\end{enumerate}
\end{theorem}

The first property above is the desired non-negativity property that we discussed earlier.
We achieve this property by providing a ``self-tensoring'' technique stated in Theorem~\ref{thm:tensoring-preserves-amm}.
The second property states our error bound.
Unlike the approximation guarantee of~\cite{choromanski2020rethinking}, though our error bound is still in terms of $\ell_2$ norms query and key vectors, it allows a larger ratio between attention weights due to the difference between exponential kernel and polynomial kernel (See more discussions in Appendix~\ref{sec:discussion_error_bound}).
The third property implies that the computation of $\phi'(\cdot)$ only requires a small number of standard matrix/vector operations which can be implemented to run quickly on accelerators (GPUs/TPUs).

Inspired by the literature of learned sketches~\cite{hsu2019learning, aamand2019learned}, we also propose a heuristic which replaces each random projection matrix used in $\phi'(\cdot)$ constructed in Theorem~\ref{thm:main_in_intro} with a comparable size learnable multi-layer dense network.
Since each random matrix used in $\phi'(\cdot)$ has size only $h\times r$ or $r\times r$, the number of parameters that we add to the model is negligible in comparison with the model size.
We observe significant model quality improvements (see Figure~\ref{fig:context-length-vs-ppl}) by learning the sketches through training instead of using randomly sampled sketches.

\paragraph{Block-based lower triangular multiplication for handling causal masks.}
Another bottleneck in applying attention linearization techniques in training transformer models with causal masking on long contexts is to handle a huge number of sequential gradients update due to RNN-style sequential state updates~\cite{hua2022transformer}.
When causal masking is applied, the attention between the query $\bq_i$ and the key $\bk_j$ is masked out when $j>i$ (i.e., the $j$-th token appears later).
More precisely, $\Attn(\bq_i, \{\bk_j\}, \{\bv_j\}) = $
\begin{align*}
\sum_{j\in [i]} \frac{\sigma(\bq_i, \bk_j)}{\sum_{j'\in [i]}\sigma(\bq_i, \bk_{j'})} \cdot \bv_j^\top
=\frac{\phi(\bq_i)^\top \cdot \sum_{j\in [i]}\phi(\bk_j)\cdot \bv_j^\top}{\phi(\bq_i)^\top \cdot \sum_{j'\in [i]}\phi(\bk_{j'})}.
\end{align*}
During the training, to compute the output of the attention mechanism in time linear in context length, one has to compute the prefix sums $\sum_{j \in [i]}\phi(\bk_j) \cdot \T{\bv_j}$ for all $i$ and then multiply the $i$-th prefix sum with the corresponding vector $\T{\phi(\bq_i)}$. This RNN-style sequential state updates make the training process fail in fully utilizing the parallelism strength of modern accelerators.
To resolve above issue, we propose a general block-based approach to compute $\lt(\bA\cdot \T{\bB})\cdot \bC$\footnote{$\lt(\bM)$ denotes the matrix obtained by only keeping the lower triangular entries of $\bM$ and zeroing the rest of the entries.} for arbitrary matrices $\bA,\bB,\bC$ without materializing $\bA\cdot \T{\bB}$, and it only requires a small number of prefix updates.
By working more carefully with our block-based approach, we observe that instead of using the approximate polynomial attention weight via approximate feature mapping, we are able to compute the exact polynomial attention weight between $\bq_i$ and $\bk_j$ if the $i$-th token and the $j$-th token are close in position.
After applying exact polynomial attention weight for local tokens, we see improvements in model qualities (see Figure~\ref{fig:context-length-vs-ppl}, Section~\ref{sec:exp} and other empirical results in Appendix).

\paragraph{Empirical studies.}
We empirically evaluate all our approaches. 
The models equipped with high degree polynomial attention and sketched polynomial attention achieve comparable or better quality on all our evaluation metrics in comparison with models equipped with vanilla softmax attention, and achieve significantly better quality than models with approximate softmax attention provided by Performer~\cite{choromanski2020rethinking}.
For GPT-2 style small size models, the models with sketched polynomial attention achieve 2x speedup in comparison with FlashAttention~\cite{dao2022flashattention,dao2023flashattention} of the fastest configuration for 32k context length.
Notice that we achieve such speed-up without applying any advanced I/O aware optimization techniques.
We believe that our running time can be further reduced by optimizing the implementation in a more careful manner.

\subsection{Other Notation}
$[n]$ denotes the set $\{1,2,3,\cdots,n\}$.
Given a matrix $\bM \in\mathbb{R}^{n \times m}$, we use $\mathbf{m}_i \in\mathbb{R}^m$ to denote the $i$-th row of $\bM$.
We also abuse the notation to use $\bM$ to indicate the set of vectors $\{\mathbf{m}_1,\mathbf{m}_2,\cdots, \mathbf{m}_n\}$.
We use $\bM_{i,j}$ to denote the entry at the $i$-th row and $j$-th column of $\bM$.
We use $\bM^p$ to indicate raising each entry of $\bM$ to the power $p$. 
Let $f:\mathbb{R}^m\rightarrow \mathbb{R}^k$ be an arbitrary function over vectors, we use $f(\bM)\in\mathbb{R}^{n\times k}$ to denote the matrix obtained by replacing the $i$-th row of $\bM$ with $f(\mathbf{m}_i)$.
$\|\bx\|_2$ denotes the $\ell_2$ norm of $\bx$, i.e., $\sqrt{\langle \bx,\bx \rangle}$.
$\frnorm{\bM}$ denotes the Frobenius norm of $\bM$, i.e., $\sqrt{\sum_{i,j} \bM_{i,j}^2}$.
Given vectors $\ba=(a_1,a_2,\cdots,a_m)$ and $\bb=(b_1,b_2,\cdots,b_m)$, $\ba * \bb = (a_1b_1,a_2b_2,\cdots,a_mb_m)$ denotes the entrywise product (Hadamard product), and $\ba \otimes \bb=(a_1b_1,a_1b_2,\cdots, a_1b_m, a_2b_1,a_2b_2,\cdots, a_2b_m,\cdots, a_mb_m)$ denotes the Kronecker product.
$\diag(\ba)$ denotes a diagonal matrix where the $i$-th diagnoal entry is $a_i$.
$\bA * \bB$ denotes the entrywise product between matrices $\bA$ and $\bB$.
We define ``self-tensoring'' $\ba^{\otimes p}:=\ba \otimes \ba^{\otimes (p-1)}\in\mathbb{R}^{m^p}$ where $\ba^{\otimes 1}:=\ba$.
$\bA^{\otimes p}$ indicates replacing each row $\ba_i$ of $\bA$ with $\ba_i^{\otimes p}$.
$\lt(\bM)$ denotes the matrix obtained by only keeping the lower-triangular entries of $\bM$ and zeroing the remaining.
$\mathbf{1}_m\in\mathbb{R}^m$ denotes an all-one vector.

\begin{figure*}[t!]
\hspace{0.4in}
    \begin{subfigure}[t]{0.35\textwidth}
        \centering
        \includegraphics[height=15em]{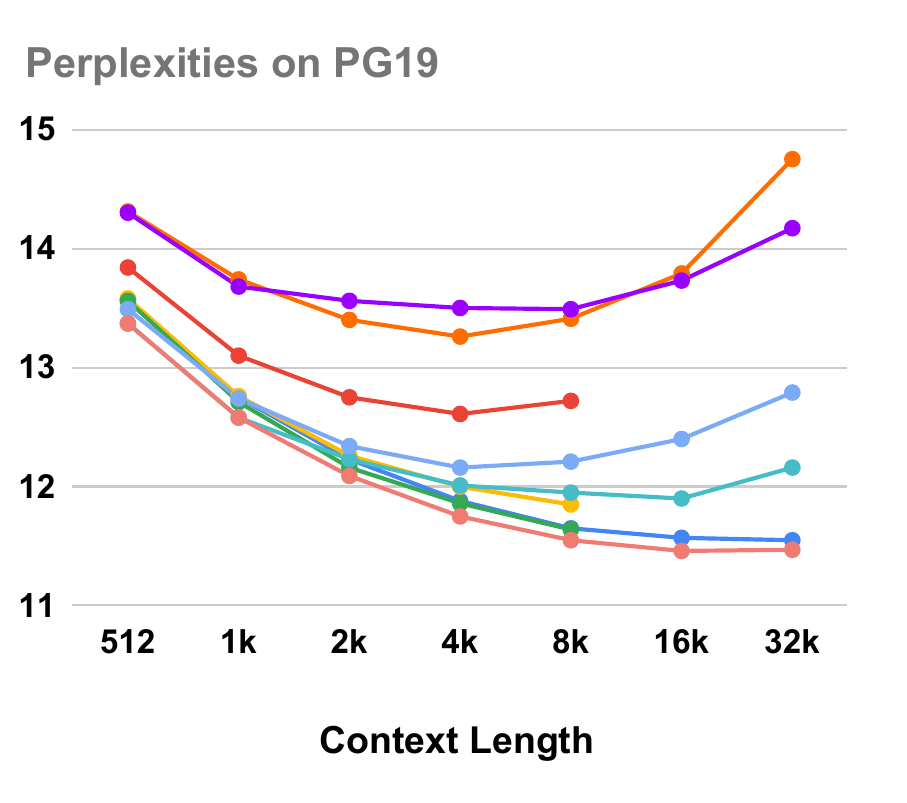}
    \end{subfigure}%
    \begin{subfigure}[t]{0.625\textwidth}
        \centering
        \includegraphics[height=15em]{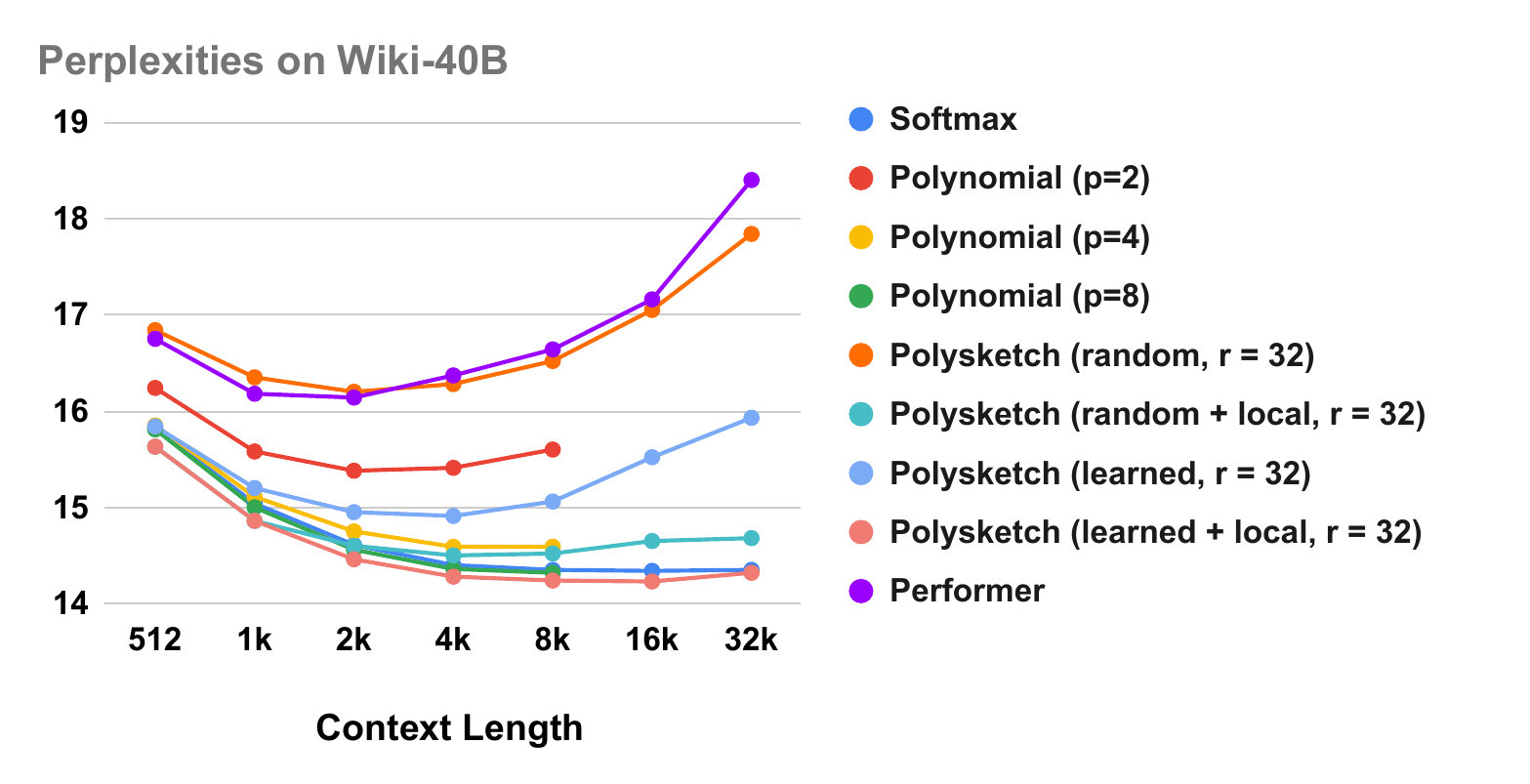}
    \end{subfigure}
    \caption{\textbf{Pre-training metric (perplexities). Lower is better.} GPT-2 small style models with various attention mechanisms are trained on PG-19 and Wiki-40B datasets at different context lengths up to 32k. 
    Each batch contains 1M tokens in total.
    Polynomial attention with $p\geq 4$ has comparable model quality as softmax attention but OOM'ed when context length >8k.
    Polysketch attention (equipped with learned sketches (Section~\ref{sec:learnable_sketch}) + local exact polynomial attention (Section~\ref{sec:local})) consistently outperforms all other mechanisms.
    The parameter $r$ denotes the sketch size (see formal definition in Section~\ref{sec:random_sketch}). See Tables~\ref{tab:ppl-pg19} and \ref{tab:ppl-wiki} in the Appendix for a full list of perplexity values.
    }
    \label{fig:context-length-vs-ppl}
\end{figure*}

\section{Polynomial Attention and Approximation}
We discuss the polynomial attention in more detail in Section \ref{sec:softmax_vs_poly} and introduce the sketching techniques (Section \ref{sec:random_sketch}, \ref{sec:learnable_sketch}) for efficiently approximating the polynomial attention.
We ignore causal masking in this section, and present how to efficiently handle causal masking in Section~\ref{sec:causal_mask}.
\subsection{Softmax versus Normalized Polynomial}\label{sec:softmax_vs_poly}
Let us revisit the softmax attention.
Given sets of query, key vectors $\bQ=\{\bq_i\}_{i\in [n]},\bK=\{\bk_i\}_{i\in [n]}\subset \mathbb{R}^h$, and scaling parameter $\beta>0$, bias parameter $\alpha\in\mathbb{R}$, the normalized softmax attention weight between $\bq_i$ and $\bk_j$ is:
\begin{align*}
\bA_{i,j} = \frac{\exp\left(\langle\bq_i, \bk_j \rangle / \beta - \alpha\right)}{\sum_{j'\in [n]} \exp\left(\langle \bq_i, \bk_{j'} \rangle/\beta - \alpha\right)}.
\end{align*}
Note $\bA_{i,j}$ is invariant in $\alpha$. 
In practice, $\alpha$ is usually chosen to be $\max_{j'\in [n]} \langle \bq_i,\bk_{j'} \rangle / \beta$ to make the computation of both numerator and denominator numerically stable.
$\beta$ is a smoothness factor.
When $\beta\rightarrow \infty$, then $\bA_{i,j}\rightarrow 1/n$, i.e., the $i$-th row of $\bA$ indicates a uniform distribution over all $j\in[n]$.
When $\beta \rightarrow 0$, then $\bA_{i,j}\rightarrow \left\{\begin{array}{ll} 0  & \langle \bq_i, \bk_j \rangle \not= \max_{j'\in [n]} \langle \bq_i, \bk_{j'} \rangle \\ 1~/~a & \langle \bq_i, \bk_j \rangle = \max_{j'\in [n]} \langle \bq_i, \bk_{j'} \rangle \end{array}\right.$ where $a$ is the number of $j$ satisfying $\langle \bq_i, \bk_j \rangle = \max_{j'\in [n]} \langle \bq_i, \bk_{j'} \rangle$, i.e., the $i$-th row of $\bA$ indicates a uniform distribution only over $j$ that provides the maximum inner product.
The standard $\beta$ is chosen to be $\sqrt{h}$~\cite{vaswani2017attention}.

Interestingly, normalized polynomial function has a similar behavior of the interpolation nature between the uniform distribution and the argmax distribution discussed above.
In particular, let $p$ be an even integer and consider the following normalized weight computed between $\bq_i$ and $\bk_j$:
\begin{align}\label{eq:raw_polynomial_attention}
\frac{\left((\langle \bq_i, \bk_j \rangle + \alpha) / \beta\right)^p}{\sum_{j'\in [n]} \left((\langle \bq_i, \bk_{j'} \rangle + \alpha)/\beta\right)^p}.
\end{align}
It is clear the weight is invariant for different $\beta$.
Choosing a proper $\beta$ can make both numerator and denominator fall in a reasonable range and make the computation numerically stable.
When $\alpha\rightarrow \infty$, the weight is close to $1/n$, i.e., these weights provide a uniform distribution.
When $\alpha \geq -\min_{j'\in[n]}\langle \bq_i,\bk_{j'}\rangle$ and $p\rightarrow \infty$, the weight is close to $0$ if $\langle \bq_i,\bk_j\rangle$ does not provide the maximum inner product, and the weight is close to $1/a$ otherwise, where $a$ is the number of $\bk_j$ that provides the maximum inner product.

Observe that if $\langle \bq_i, \mathbf{1}_h \rangle=\langle \bk_{j}, \mathbf{1}_h \rangle=0$ for all $i,j\in [n]$, i.e., entries of $\bq_i,\bk_{j}$ always have mean $0$, then we have $\forall i,j\in[n],$
$
 (\langle \bq_i, \bk_{j} \rangle + \alpha) / \beta = \left\langle \bq'_i, \bk'_{j}\right\rangle
$,
where $\bq'_i = \bq_i / \sqrt{\beta} + \sqrt{\alpha/(\beta h)}\cdot \mathbf{1}_h$, and $\bk'_{j} = \bk_{j} / \sqrt{\beta} + \sqrt{\alpha/(\beta h)}\cdot \mathbf{1}_h$, i.e., $\bq'_i$ and $\bk'_j$ are obtained by applying the same rescaling and bias to $\bq_i$ and $\bk_j$ respectively.
Motivated by the above observation, we slightly tweak Equation~\ref{eq:raw_polynomial_attention} by applying an additional layer normalization\footnote{Layer normalization shifts the entries of the input vector to make them have mean $0$ and learns a suitable bias during training.}~\cite{ba2016layer} to $\{\bq_i\},\{\bk_j\}$, this gives the normalized degree-$p$ polynomial attention weight matrix $\bA^{(p)}$ considered in this paper:
\begin{align*}
\bA^{(p)}_{i,j} = \frac{\langle \bq'_i, \bk'_j\rangle^p}{1 + \sum_{j'\in [n]} \langle \bq'_i,\bk'_{j'}\rangle^p}
\end{align*}
where $\bq'_i,\bk'_j$ are obtained by applying the layer normalization layer to $\bq_i,\bk_j$ respectively.
Unlike softmax attention matrix, it is possible that the term $\sum_{j'\in [n]} \langle \bq'_i, \bk'_{j'} \rangle^p$ is (close to) $0$.
We add $1$ to the denominator to avoid dividing by zero. Given value vectors $\bV = \{\bv_i\}_{i\in[n]}\subset \mathbb{R}^h$, the full degree-$p$ polynomial attention $\Attn^{(p)}(\bQ,\bK,\bV)=\bA^{(p)} \cdot \bV=\bD^{-1}\cdot (\bQ'{\bK'}^{\top})^p\cdot \bV$, where $\bD=\diag(\mathbf{1}_n+(\bQ'{\bK'}^{\top})^p \mathbf{1}_n)$.
In the rest of the paper, we abuse notation between $\bQ,\bK$ and $\bQ',\bK'$, and only consider $\bQ,\bK$ after layer normalization.

As presented in Figure~\ref{fig:context-length-vs-ppl} and other experiments in Section~\ref{sec:exp} and Appendix, the models with the degree-$p$ polynomial attention described above achieve comparable model quality as vanilla softmax attention on all metrics as long as $p\geq 4$. To test the long range learning capabilities and in-context learning capabilities of attention mechanisms, we study the synthetic tasks of Selective Copying \cite{gu2023mamba} and Induction heads \cite{olsson2022context}. 
The models with polynomial attention for $p \ge 4$ perform as well as models with softmax attention (see Appendix~\ref{apx:synthetic} for more details). 
    \subsection{Random Sketches for Polynomial Attention with Theoretical Gaurantees}\label{sec:random_sketch}
To compute $\Attn^{(p)}(\bQ, \bK, \bV)$, we only need to compute $(\bQ\bK^\top)^p\cdot\bV$ and $(\bQ\bK^\top)^p\cdot \mathbf{1}_n$.
Let us only focus on computing $(\bQ\bK^\top)^p\cdot\bV$ since we can handle $(\bQ\bK^\top)^p\cdot \mathbf{1}_n$ in the same way.
Due to a well-known fact $\forall \bx,\by, \langle \bx,\by\rangle^p = \langle \bx^{\otimes p}, \by^{\otimes p}\rangle$, we have $(\bQ\bK^\top)^p\bV = \bQ^{\otimes p} (\bK^{\otimes p})^\top\cdot \bV$.
If we reorder the computation and compute $(\bK^{\otimes p})^\top\cdot \bV$ first, we are able to compute $\bQ^{\otimes p} \cdot (\bK^{\otimes p})^\top\cdot \bV$ in $O(nh^{p+1})$ time which is linear in the context length $n$. 
However $h^{p+1}$ dependence is still expensive as we explained in Section~\ref{sec:our_contribution}.
Thus, we resort to approximating $\bQ^{\otimes p} (\bK^{\otimes p})^\top$ using sketching techniques, which we formally describe ahead. We first state the definition of a sketch that has the ``Approximate Matrix Multiplication (AMM)'' guarantee.
\begin{definition}[Approximate Matrix Multiplication~\cite{woodruff2014sketching}]
    \label{def:amm}
    Given parameters $n$, $h$ and $p$, a randomized sketching matrix $\bS \in \R^{h^p \times r}$ has the $(\varepsilon,p)$-AMM property if given any two $n \times h$ matrices $\bA$ and $\bB$, with probability $\ge 9/10$  over the randomized sketching matrix $\bS$, we have that
    $
        \frnorm{(\bA^{\otimes p}\bS)\T{(\bB^{\otimes p}\bS)} - \bA^{\otimes p}\T{(\bB^{\otimes p})}} \le \varepsilon\frnorm{\bA^{\otimes p}}\frnorm{\bB^{\otimes p}}.
    $ 
\end{definition}
The parameter $r$ above is referred to as the \emph{sketch size}. 
Two important properties of a sketching distribution are (i) the sketch size $r$ as a function of the accuracy parameter $\varepsilon$ and (ii) the time required to compute $\bA^{\otimes p}\bS$ given an arbitrary matrix $\bA$. Ideally, we want the matrix $\bS$ to have a structure such that $\bA^{\otimes p}\bS$ can be computed without realizing the large matrix $\bA^{\otimes p}$. \citet{ahle2020oblivious} gave constructions of different sketches that have both the properties that the sketch size $r$ is small and the matrix $\bA^{\otimes p}\bS$ can be computed quickly. We describe the main properties of one of their sketches below and explain how to compute $\bA^{\otimes p}\bS$.
\begin{theorem}[\cite{ahle2020oblivious}]\label{thm:previous_sketch}
    Given $p$ and $\varepsilon$, there is a sketching matrix $\bS$ with $r =  \Theta(p/\varepsilon^2)$ columns such that $\bS$ satisfies the $(\varepsilon, p)$-AMM property (Definition~\ref{def:amm}). Given an arbitrary vector $\ba \in \R^{h}$, computing $(\ba^{\otimes p})^\top\bS$ only requires $p$ matrix-vector multiplications of matrix size $h\times r$, $(p-2)$ matrix-vector multiplications of matrix size $r\times r$, and $(p-1)$ Hadamard products of $r$-dimensional vectors.
\end{theorem}
To compute $\bA^{\otimes p}\bS$, we only need to compute $(\ba_i^{\otimes p})^\top \bS$ for each row $\ba_i$ of $\bA$. The number of matrix-vector multiplications and Hadamard products scales linearly in $n$.
Let us focus on the construction of the sketch described in Theorem~\ref{thm:previous_sketch}.
We now explain how the sketch computation works for $p=2$ and how it is extended to general values of $p$ that are powers of $2$.
Let $\bG_1\in \mathbb{R}^{h\times r}$ and $\bG_2\in\mathbb{R}^{h\times r}$ denote two independently sampled random Gaussian matrices, i.e., each entry is drawn indepenently from a standard Gaussian distribution.
Then the outcome of applying the sketch on $\bA^{\otimes 2}$ is $\bA^{\otimes 2}\bS=\sqrt{1 / r} \cdot [(\bA\bG_1) * (\bA\bG_2)]$.
The construction extends to all $p$ that are powers of $2$ in a recursive way. $\textsc{PolySketchWithNegativity}(\bA,r,p)$ (Algorithm~\ref{alg:polysketch}) shows how to compute $\bA^{\otimes p}\bS$ in general.

\begin{algorithm}[t]
\small
\caption{Polynomial Sketches}
\label{alg:polysketch}
\begin{algorithmic}
\FUNCTION{\textsc{PolySketchWithNegativity}$(\bA\in\mathbb{R}^{k\times m}, r, p)$}
\STATE // Implementation of Theorem~\ref{thm:previous_sketch}~\cite{ahle2020oblivious}.
\STATE // The output computes $\bA^{\otimes p}\bS$.
\STATE If $p=1$, return $\bA$
\STATE $\bM_1$ = \textsc{PolySketchWithNegativity$(\bA, r, p/2)$}
\STATE $\bM_2$ = \textsc{PolySketchWithNegativity$(\bA, r, p/2)$}
\STATE Sample Gaussian matrices $\bG_1,\bG_2$, each of $r$ columns
\STATE Return $\sqrt{1/r}\cdot [(\bM_1\bG_1) *  (\bM_2\bG_2)]\in\mathbb{R}^{k\times r}$
\ENDFUNCTION
\FUNCTION{\textsc{PolySketchNonNegative}$(\bA\in\mathbb{R}^{k\times m}, r, p)$}
\STATE // Our approach based on Theorem~\ref{thm:tensoring-preserves-amm}.
\STATE // The output computes $\phi'(\bA)=(\bA^{\otimes (p/2)}\bS)^{\otimes 2}$ where $\phi'(\cdot)$ is the same mapping as mentioned in Theorem~\ref{thm:main_in_intro}.
\STATE $\bM$ = \textsc{PolySketchWithNegativity$(\bA, r, p/2)$}
\STATE Return $\bM^{\otimes 2}\in\mathbb{R}^{k\times r^2}$.
\ENDFUNCTION
\end{algorithmic}
\end{algorithm}
The polynomial sketch described above can be used to approximate the matrix $(\bQ \T{\bK})^p = \bQ^{\otimes p}\T{(\bK^{\otimes p})}$ with $(\bQ^{\otimes p}\bS)\T{(\bK^{\otimes p}\bS)}$. 
However one issue is that they do not preserve nonnegativity: while for even $p$, the entries of the matrix $(\bQ \T{\bK})^p$ are  nonnegative, the entries of the matrix $(\bQ^{\otimes p}\bS)\T{(\bK^{\otimes p}\bS)}$ can be negative.
This is not desired as discussed in Section~\ref{sec:our_contribution}.
In the following, we propose a novel but simple approach to address this negativity issue.

Consider two arbitrary vectors $\ba, \bb$, we can see that the dot product $\la \ba^{\otimes 2}, \bb^{\otimes 2}\ra = \la \ba, \bb\ra^2 \ge 0$. Thus, given matrices $\bQ^{\otimes (p/2)}\bS$ and $\bK^{\otimes (p/2)}\bS$, consider the matrix $(\bQ^{\otimes (p/2)}\bS)^{\otimes 2}\T{((\bK^{\otimes (p/2)}\bS)^{\otimes 2})}$. Since all the entries of the matrix are of the form $\la \ba^{\otimes 2}, \bb^{\otimes 2}\ra$ for some vectors $\ba$, $\bb$, all the entries of the matrix $(\bQ^{\otimes (p/2)}\bS)^{\otimes 2}\T{((\bK^{\otimes (p/2)}\bS)^{\otimes 2})}$ are nonnegative as well. The ``self-tensoring'' trick ensures that all the entries in the approximate attention matrix are nonnegative at the cost of \emph{squaring} the sketch size $r$.

Although $(\bQ^{\otimes (p/2)}\bS)^{\otimes 2}\T{((\bK^{\otimes (p/2)}\bS)^{\otimes 2})}$ guarantees non-negative property, it is not clear whether it is still a good approximation to $(\bQ\bK^{\top})^p$ given that $\bS$ is a polynomial sketch for degree $p/2$.
One of our technical contributions is to provide a non-trivial analysis to show that it still provides a good approximation when the sketching matrix $\bS$ is constructed as in \cite{ahle2020oblivious}.
The key is Theorem~\ref{thm:tensoring-preserves-amm} which shows that a degree $p/2$ polynomial sketch followed by ``self-tensoring'' gives a degree $p$ polynomial sketch. 

To state Theorem~\ref{thm:tensoring-preserves-amm} properly, we need to briefly introduce following concepts.
The $(\varepsilon, \delta, t)$-JL moment property is defined as follows. 
Given a scalar random variable $\bX$ and $t \ge 1$, $\|\bX\|_{L^t}$ is defined to be $\E[|\bX|^t]^{1/t}$. $\|\cdot\|_{L^t}$ defines a norm over random variables defined over the same sample space and in particular satisfies $\|\bX + \bY\|_{L^t} \le \|\bX\|_{L^t} + \|\bY\|_{L^t}$.
\begin{definition}[JL-moment property~\cite{woodruff2014sketching}] Given $\varepsilon, \delta \ge 0,t \ge 1$, a random matrix $\bS^{m \times r}$ has the $(\varepsilon, \delta, t)$-JL moment property if for any $\bx \in \R^m$ with $\opnorm{\bx}=1$,
$
   \left\|\opnorm{\T{\bx}\bS}^2 - 1\right\|_{L^t} \le \varepsilon\cdot\delta^{1/t}.
$
\end{definition}

\begin{theorem}
Let $\bS \in \mathbb{R}^{h^{p/2} \times r}$ be a random sketch  satisfying the $(\varepsilon, \delta, t)$-JL moment and $(\varepsilon, \delta, 2t)$-JL moment properties for some even integer $t$. Given matrices $\bC,\bD$ with $h^{p/2}$ columns,
$
    \frnorm{(\bC\bS)^{\otimes 2}\T{((\bD\bS)^{\otimes 2})} - \bC^{\otimes 2}\T{(\bD^{\otimes 2})}} \le \sqrt{5}\varepsilon\frnorm{\bC^{\otimes 2}}\frnorm{\bD^{\otimes 2}}
$
 holds with probability $\ge 1 - \delta$,
\label{thm:tensoring-preserves-amm}
\end{theorem}
Due to the page limit, we defer the proof to Appendix~\ref{sec:proof_of_tensoring}.
\paragraph{Proof of Theorem~\ref{thm:main_in_intro}.}
Results from Section~4 of \cite{ahle2020oblivious} implies that the polynomial sketch $\bS$ as mentioned in Theorem~\ref{thm:previous_sketch} for degree $p/2$ with sketch size $r = \Theta(p/\varepsilon^2)$ satisfies the requirements of Theorem~\ref{thm:tensoring-preserves-amm}.
By plugging $\bQ^{\otimes (p/2)}$, $\bK^{\otimes (p/2)}$ into $\bC,\bD$ of Theorem~\ref{thm:tensoring-preserves-amm} respectively and scaling $\varepsilon$ properly, we obtain $\frnorm{(\bQ^{\otimes (p/2)}\bS)^{\otimes 2}\T{((\bK^{\otimes (p/2)}\bS)^{\otimes 2})} - (\bQ\bK^\top)^p}\leq \varepsilon\frnorm{\bQ^{\otimes p}}\frnorm{\bK^{\otimes p}}$ which concludes Theorem~\ref{thm:main_in_intro}, i.e., the approximate feature mapping $\phi'(\bx)=((\bx^{\otimes (p/2)})^{\top}\bS)^{\otimes 2}\in\mathbb{R}^{r^2}$ and $\phi'(\bQ),\phi'(\bK)$ can be efficiently computed using \textsc{PolySketchNonNegative}$(\cdot,r,p)$ (see Algorithm~\ref{alg:polysketch}).

Using $\phi'(\cdot)$,
we get the following approximate polynomial attention $\widetilde{\Attn}^{(p)}(\bQ,\bK,\bV)= \tilde{\bD}^{-1} \phi'(\bQ) \phi'(\bK)^\top \bV$, where $\tilde{\bD}=\diag(\mathbf{1}_n+\phi'(\bQ) \phi'(\bK)^\top \mathbf{1}_n)$.
We call this attention mechanism Polysketch attention.
    \subsection{Learnable Sketches for Polynomial Attention}\label{sec:learnable_sketch}
There are only $(p-2)$ random projections where each is introduced by a matrix multiplication with a small Gaussian matrix ($\bG_1,\bG_2$ in each recursion call in Algorithm~\ref{alg:polysketch}) of size either $h\times r$ or $ r\times r$ during the recursive computation of $\phi'(\bX)$ = \textsc{PolySketchNonNegative}$(\bX, r, p)$ (Algorithm~\ref{alg:polysketch}) for $\bX\in\mathbb{R}^{n\times h}$. 
Inspired by the literature of learned sketches~\cite{hsu2019learning,aamand2019learned}, a natural idea is to replace each random matrix $\bG_1$, $\bG_2$ in Algorithm~\ref{alg:polysketch} with learnable parameters. 
In practice, we found that replacing each of these random projections with a learnable \emph{non-linear} transformation introduced by a dense neural network with size comparable to $\bG_1,\bG_2$ achieves a better model quality.
We describe more details of our network structure for the learnable non-linear transformation in Appendix~\ref{sec:network_structure_for_learnable_sketch}.
We also evaluate models with Polysketch attention with learned sketches on induction heads and selective copying synthetic tasks. We find that the models perform as well as models with softmax attention (See Appendix~\ref{apx:synthetic}).

\section{Dealing with Causal Masks}\label{sec:causal_mask}
When considering causal masks, the Polysketch attention with respect to $\bq_i$ is 
$
\underset{j\leq i}{\sum}\frac{ \langle \phi'(\bq_i), \phi'(\bk_j) \rangle}{1 + \sum_{j'\leq i} \langle \phi'(\bq_i), \phi'(\bk_{j'}) \rangle}\cdot \bv_j^\top.
$
In this causal case, the full Polysketch attention can be written as 
$
\widetilde{\Attn}^{(p)}(\bQ, \bK, \bV) = \tilde{\bD}^{-1}\cdot \lt(\phi'(\bQ)\phi'(\bK)^\top)\cdot \bV
$ where $\tilde{\bD}= \diag(\mathbf{1}_n + \lt(\phi'(\bQ)\phi'(\bK)^\top)\cdot \mathbf{1}_n)$.
Therefore, it is crucial to efficiently compute $\lt(\phi'(\bQ)\phi'(\bK)^\top)\cdot \bX$ for $\bX\in \{\mathbf{1}_n, \bV\}$.
In the next subsection, we present a block based algorithm to compute $\lt(\bA\cdot \bB)\cdot \bC$ for arbitrary matrices $\bA,\bB\in\mathbb{R}^{n\times m}, \bC\in\mathbb{R}^{n\times k} $ in time linear in $n$, while the number of sequentially dependent steps is small.
    \subsection{Fast Lower Triangular Multiplication}\label{sec:fast_lt}
\begin{figure}[t]
    \centering
\includegraphics[width=1.1\linewidth]{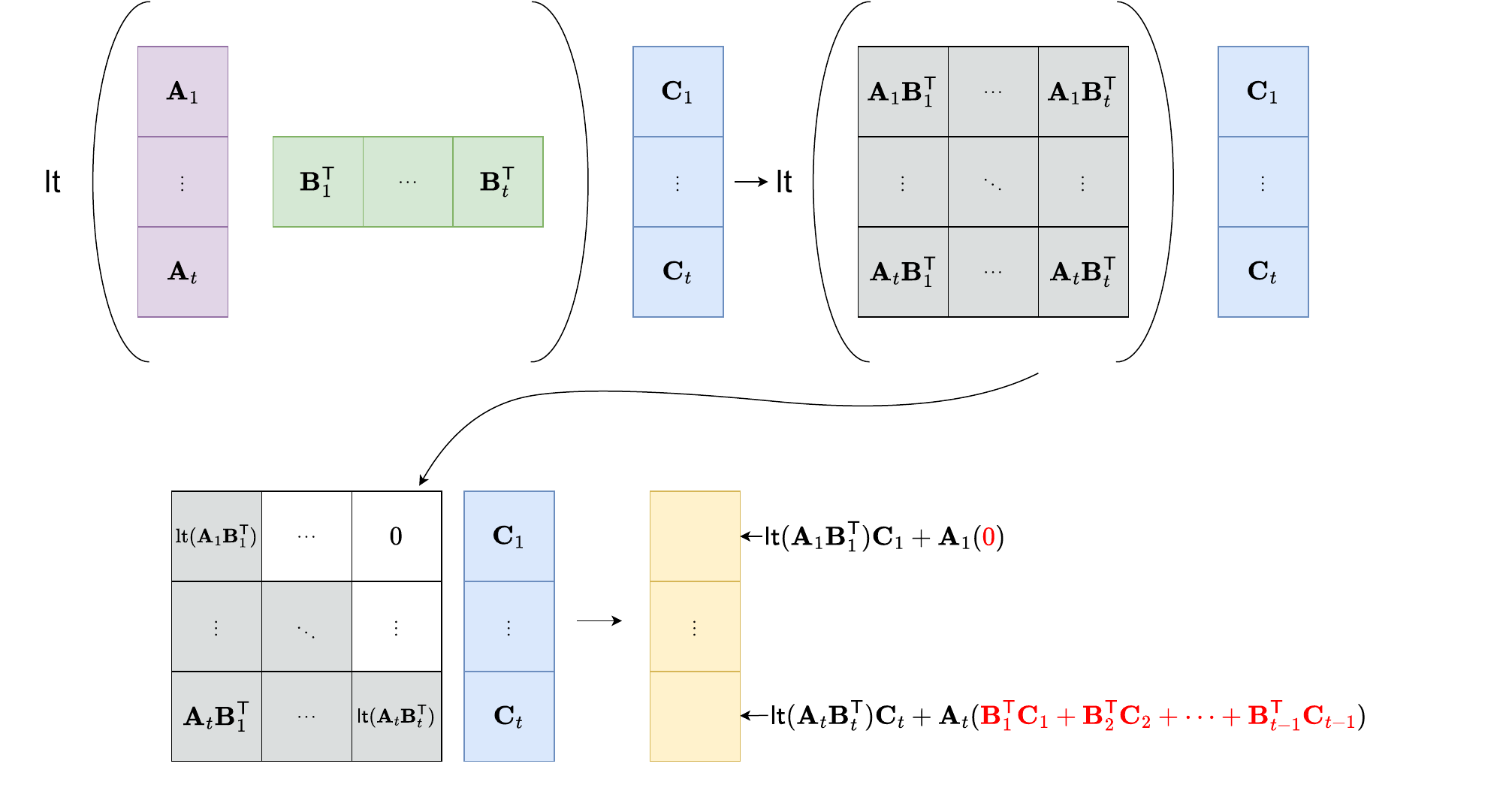}
\caption{
\textbf{Block wise Lower Triangular Multiplication.} $\bA_l$, $\bB_l$, $\bC_l$ are blocks of $\bA$, $\bB$, $\bC$. Each block has $b = n/t$ rows.}
    \label{fig:lt-multiply-drawio}
\end{figure}
Let $b$ be the block size and $t = n/b$ be the number of blocks where each block $B_l$ $(l\in [t])$ contains indices $\{(l-1) b+1,(l-1) b+2,\cdots,l\cdot b\}$.
Let $\ba_i,\bb_i,\bc_i$ denote the $i$-th row vector of $\bA,\bB,\bC$ respectively.
For each $l\in[t]$, let the rows of $\bA_l\in\mathbb{R}^{b\times m}$ consist of $\ba_i$ where $i\in B_l$.
We define sub-matrices $\bB_l,\bC_l$ of $\bB,\bC$ respectively in a similar way.
For $l\in [t]$, let us compute $\bH_l = \sum_{i\in B_l} \bb_i \bc_i^\top$. 
Let $\bZ_l$ indicates the prefix sum: $\bZ_l=\sum_{j< l} \bH_j$.
In addition, let us compute $\bP_l = \lt(\bA_l\bB_l^\top) \bC_l$ for each $l\in [t]$ in the direct way.
For any $l\in [t]$, and any $i\in B_l$, if $i$ is the $i'$-th index within the block $B_l$, it is easy to verify that the $i$-th row of $\lt(\bA\bB^\top)\bC$ can be obtained by $\bp +\ba_i^\top \bZ_{l}$ where $\bp$ is the $i'$-th row of $\bP_l$.
Figure~\ref{fig:lt-multiply-drawio} justifies the correctness of the above algorithm.

Since the prefix sum $\bZ_t$ is over $t$ matrices, the number of sequentially dependent steps is $t$. We can further reduce the number of sequential steps by using a parallel prefix sum algorithm \cite{blelloch1990prefix} to exploit the parallelism. In our implementation, we only use the sequential prefix sum algorithm.
Computing all $\bP_l$ requires $O(t\cdot b^2(m + k))$ time.
Computing all $\bH_l,\bZ_l$ requires $O(t\cdot bmk)$ time.
Computing all $\bA_l \bZ_l+\bP_l$ requires $O(t\cdot bmk)$ time.
Therefore, the overall running time is $O(nb(m+k))$.
When we set $b$ as a constant, the running time is linear in $n$. 
If we directly plugging $\phi'(\bQ), \phi'(\bK), \bV$ (or $\mathbf{1}_n$) into $\bA,\bB,\bC$ above respectively, we compute causal Polysketch attention in time linear in the context length, $n$.

Let us take another look using the above process to compute Polysketch attention, the matrix $\bP_l$ actually corresponds to $\lt(\phi'(\bQ)_l \phi'(\bK)_l^\top) \bX_l$ where $\bX_l\in\{\bV_l, \mathbf{1}_b\}$.
$\phi'(\bQ)_l$, $\phi'(\bK)_l$ corresponds to approximate feature mapped query and key vectors within the block $B_l$, and $\bV_l$ corresponds to the value vectors in $B_l$.
Let $\bQ_l,\bK_l$ be the corresponding original query and key vectors in $B_l$.
One observation is that $\phi'(\bQ)_l \phi'(\bK)_l^\top = \bL^{\otimes 2} (\bR^{\otimes 2})^\top = \left(\bL \bR^\top\right)^2$ where $\bL=\textsc{PolySketchWithNegativity}(\bQ_l,r,p/2)\in\mathbb{R}^{b\times r}$ and $\bR=\textsc{PolySketchWithNegativity}(\bK_l,r,p/2)\in\mathbb{R}^{b\times r}$ (recall Algorithm~\ref{alg:polysketch}).
Therefore $\lt(\phi'(\bQ)_l \phi'(\bK)_l^\top)$ only takes $O(b^2 r)$ time instead of $O(b^2r^2)$ time.
The total time to compute Polysketch attention is  $O(nb(r+h)+nr^2h).$

\subsection{Applying Exact Attention Locally}\label{sec:local}
We further observe that $\phi'(\bQ)_l \phi'(\bK)_l^\top$ is used to approximate $(\bQ_l\bK_l^\top)^p$.
We can actually compute $\bP_l$ as $\lt\left((\bQ_l\bK_l^\top)^p\right)\bX_l$.
This means that when token $i$ and $j$ are within the same local block, we can use their exact polynomial attention weight instead of using the approximation.
The time to compute $\lt\left((\bQ_l\bK_l^\top)^p\right)\bX_l$ is at most $O(b^2h)$.
In this case, the total time to compute our Polysketch attention is at most $O(nh(b+r^2))$.
When $b\leq r^2$, the running time is $O(nhr^2)$.
As observed by our empirical studies (see Figure~\ref{fig:context-length-vs-ppl}, Section~\ref{sec:exp} and other experiments in the appendix), using exact polynomial attention weights inside each local block further improves the model quality.

\section{Experiments}\label{sec:exp}

To evaluate the effectiveness of the polynomial attention and Polysketch attention mechanisms, we train language models of various sizes with different attention mechanisms and look at both pre-training metrics and the performances on downstream tasks.
Our implementations of all models are written in JAX. 
In our experiments, we use a Pallas implementation \citep{pallas_flash} of FlashAttention and a JAX implementation of Performer open-sourced by the authors \citep{choromanski2020rethinking}. All the experiments are conducted on 32 Google Cloud TPUs.

\begin{table*}
\small
\centering
\begin{tabular}{l c c c c c c c}\toprule
&  C4& \multicolumn{2}{c}{HellaSwag} & \multicolumn{2}{c}{PIQA} & \multicolumn{2}{c}{Physics}\\
\cmidrule(lr){2-2}\cmidrule(lr){3-4}\cmidrule(lr){5-6}\cmidrule(lr){7-8}
& Perplexity $\downarrow$    & 0-shot $\uparrow$ & 5-shot $\uparrow$ & 0-shot $\uparrow$ & 5-shot $\uparrow$ & 0-shot $\uparrow$ & 5-shot $\uparrow$\\\midrule
\multicolumn{6}{l}{\textbf{GPT-2 Small style, 100M-scale, 12 layers default, Context Length 8192, 125k training steps}}\\
Softmax                                      &  17.81 & {30.2} & 27.8 & {64.6} & 63.2 & 27.5 & 27.5\\
Polynomial (degree 4)                        &  18.18 & 28.6        & \underline{28.4} & 64.2 & \textbf{\underline{65.0}} & \underline{27.5} & \underline{31.0}\\
Polynomial (degree 8)                        &  \underline{17.77} & 29.8         & \underline{29.8} & 62.2 & \underline{64.0} & 23.1 & 26.2\\
Polysketch (learned, r = 64)                         &  18.79 & 29.6          & \underline{28.6} & 60.0 & 60.0 & 24.8 & \underline{30.5}\\
Polysketch (learned, 13 layers, r = 64)              &  18.47 & 28.4 & \underline{29.4} & 62.0 & 62.6 & \underline{27.5} & \underline{31.8}\\
Polysketch (learned + local, r = 64)                 &  17.98 & 29.8 & \textbf{\underline{30.6}} & 62.4 & \underline{63.6} & \textbf{\underline{30.1}} & {\underline{32.3}}\\
Polysketch (learned + local, 13 layers, r = 64)      &  \textbf{\underline{17.68}} & 29.0 & \underline{29.0} & 62.6 & \underline{64.2} & 20.5 & 27.0\\
Polysketch (learned, r = 32)                &   19.09   &   28.0   &    28.4   &    60.6    &   62.0    &   28.3    &   27.5      \\
Polysketch (learned, 13 layers, r = 32)     &   19.50   &   28.4    &   29.0   &    61.6    &   \underline{64.6}    &   27.9    &   \underline{33.1}    \\
Polysketch (learned + local, r = 32)        &   18.04   &   29.0    &   29.2    &   63.4    &   62.8    &   26.6    &   \underline{\textbf{35.8}}    \\
Polysketch (learned + local, 13 layers, r = 32) & \underline{17.72} &   \underline{\textbf{31.2}}    &   \underline{30.4}    &   \underline{\textbf{64.8}}    &   \underline{64.6}    &   \underline{27.9}    &   \underline{31.8}\\\midrule
\multicolumn{6}{l}{\textbf{GPT-2 Medium style, 300M-scale, 24 layers default, Context Length 8192, 125k training steps}}\\
Softmax                                      &  \textbf{13.98} & 35.8 & \textbf{36.6} & {67.0} & 67.2 & 30.5 & 25.7\\
Polynomial (degree 4)                        &  14.29 & \underline{35.8} & 36.0 & 65.8 & \underline{67.6} & 27.5 & \underline{28.8}\\
Polynomial (degree 8)                        &  14.14 & {\underline{37.0}} & \textbf{\underline{36.6}} & 65.4 & 65.6 & \textbf{\underline{33.1}} & \underline{27.5}\\
Polysketch (learned, r = 64)                         &  14.64 & 34.6 & 33.4 & 63.2 & 65.4 & \underline{31.0} & \underline{26.2}\\
Polysketch (learned, 26 layers, r = 64)              &  14.49 & 34.8 & 34.4 & 65.2 & 66.6 & 28.4 & {24.9}\\
Polysketch (learned + local, r = 64)                 &  14.16 & 35.0 & 35.0 & 65.8 & \textbf{\underline{68.6}} & 29.6 & \textbf{\underline{34.5}}\\
Polysketch (learned + local, 26 layers, r = 64)      &  \textbf{\underline{13.98}} & \underline{35.8} & 35.4 & 66.4 & \textbf{\underline{68.6}} & 27.0 & {\underline{33.6}}\\
Polysketch (learned, r = 32)                &   14.94   &   32.2    &   33.8    &   65.6    &   \underline{67.6}    &   \underline{32.7}    &   \underline{33.6}   \\
Polysketch (learned, 26 layers, r = 32)     &   14.73   &   32.8    &   35.2    &   65.0    &   65.2    &   28.3    &   \underline{31.8}   \\
Polysketch (learned + local, r = 32)        &   14.15   &   \underline{36.0}    &   35.8    &   65.2    &   \underline{67.6}    &   27.5    &   \underline{27.9}    \\
Polysketch (learned + local, 26 layers, r = 32) &  14.00    &   \underline{\textbf{37.2}}    &   35.4    &   \underline{\textbf{68.0}}    &   \underline{67.6}    &   23.1    &   \underline{29.6}\\\midrule
\multicolumn{6}{l}{\textbf{GPT-2 Large style, 700M-scale, 36 layers default, Context Length 2048, 125k training steps}}\\
Softmax                & 12.71                    &  40.2 & 40.2 & 68.8 & \textbf{71.4} & 34.4 & 24.4 \\
Polynomial (degree 4)&   12.82                    & 40.0 & \textbf{\underline{40.6}} & 67.8 & 66.6 & 31.8& \underline{31.4} \\
Polynomial (degree 8)&   12.85                    & 40.0 & 39.8 & 66.8 & 70.4 & \underline{34.4} & \underline{29.6}\\
Polysketch (learned, 39 layers, r = 64)& 12.83            & \textbf{\underline{41.0}} & 39.4 & 68.6 & 68.8 & 33.6 & \underline{36.6} \\
Polysketch (learned + local, 39 layers, r = 64) & \textbf{\underline{12.70}}   & \underline{40.6} & 40.0 & \textbf{\underline{69.0}} & 69.0 & \textbf{\underline{38.4}} & \textbf{\underline{37.1}} \\
Polysketch (learned, 39 layers, r = 32) &   12.98   &   39.4    &   \underline{40.4}    &   68.6  & 67.6    &   33.6     &  27.0   \\
Polysketch (learned + local, 39 layers, r = 32)&  12.74 &   39.6    &   \underline{\textbf{40.6}}    &   66.8    &   69.4    &   \underline{35.3}    &   \underline{31.8} \\\bottomrule
\end{tabular}
\caption{We compare the accuracies(\%, higher the better) of different models on three different Q/A tasks. HellaSwag and Physics tasks have 4 choices and PIQA task has 2 choices. We also report the perplexities (lower the better) on the validation split of C4 dataset. 
\textbf{Bolding} indicates the best model in the task, \underline{underlining} indicates beating softmax attention.}
    \label{tab:main-table}
\end{table*}

\paragraph{Synthetic tasks.}
Selective Copying and Induction Heads are two well-known downstream synthetic tasks for measuring content aware reasoning capabilities and the memorization abilities of the models (see \citet{gu2023mamba} for more discussions).
We conduct both experiments and see both polynomial and Polysketch have similar performance as softmax attention. We include more details in Appendix~\ref{apx:synthetic}.

\begin{figure*}
    \centering
\includegraphics[width=0.7\linewidth]{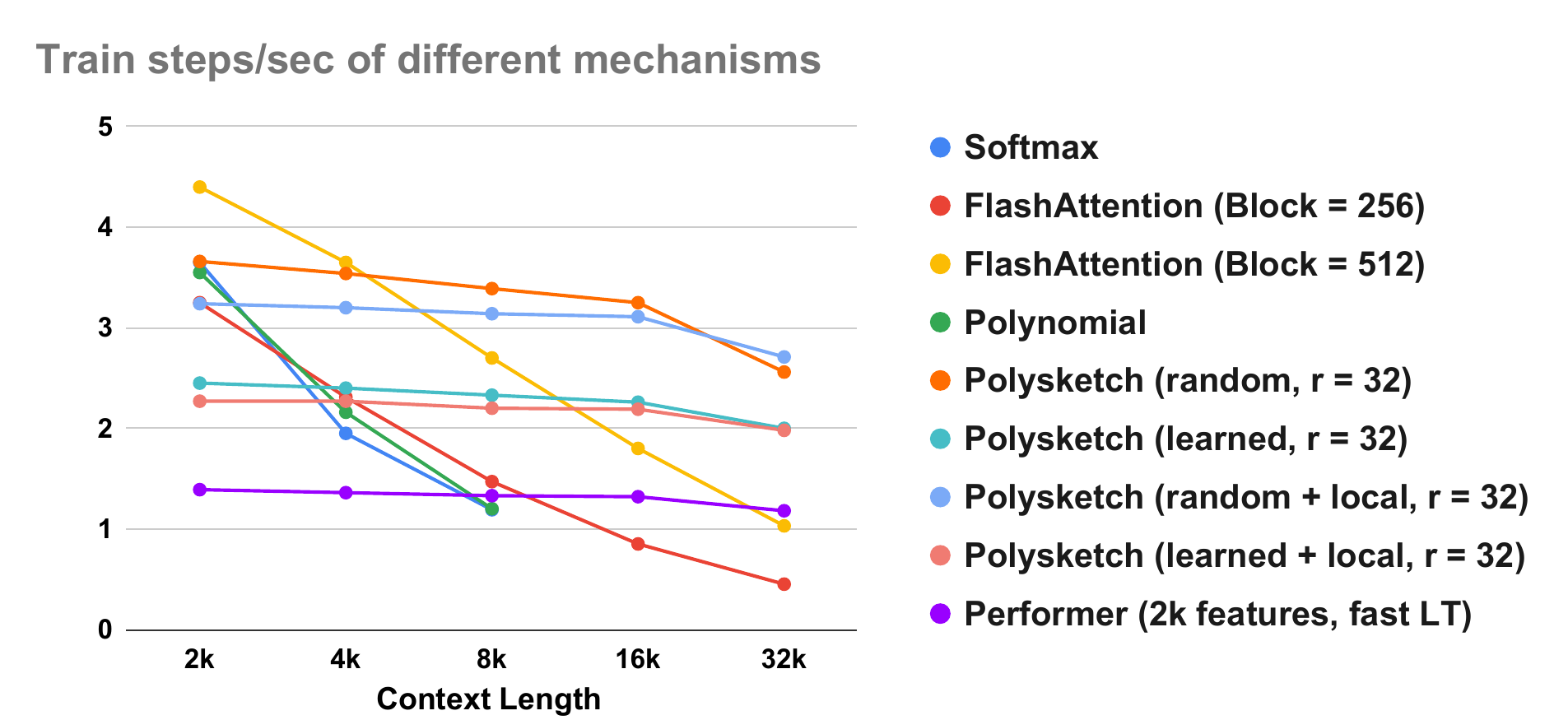}
\caption{\small
Training speed of models on PG-19 and Wiki-40B for different context lengths.
Softmax and polynomial attentions OOM'ed when context length >8k.}
    \label{fig:training_latency}
\end{figure*}

\paragraph{Models.} 
For real world datasets, we train decoder-only models (only contain causal masked attention layers) of three different scales, mirroring the GPT-2 family~\cite{radford2019language}: Small, Medium and Large. 
For small scale models, we train models using context lengths from 512 to 32k.
For medium scale models, we only train using context length 8k.
For large scale models, we only train using context length 2k.
The reason that we did not train longer context length for medium and large scale models is that non-kernel based attention mechanisms (softmax, polynomial) are too slow or go out of memory (OOM).
 The detailed descriptions of model sizes can be found in Appendix~\ref{sec:model_sizes}.
 We take the recipe of Transformer++ (see \cite{hua2022transformer,yang2023gated,gu2023mamba} as well).
 We refer readers to Appendix~\ref{sec:transformer_pp} for a detailed description of the Transformer++ used by us.
 If not specified otherwise, we use 10k warmup steps, 125k total training steps and a linear learning rate schedule.
Depending on the original model scale, we also train kernel based attention models (Polysketch and Performer) with $0-3$ additional layers, since these models are significantly faster than non-kernel based attention models so we can afford to train larger models compared to vanilla softmax. It only slightly increases model sizes.

\paragraph{Attention Mechanisms.} We train models with the following $4$ categories of attention mechanisms: (i) Softmax, (ii) Polynomial ($p=2 ,4, 8$), (iii) Polysketch (approximating polynomial attention of $p=4$) with variants enabling learned sketches (Section~\ref{sec:learnable_sketch}) or local exact polynomial attention (Section~\ref{sec:local}) or both, and (iv)  Performer equipped with our lower triangular multiplication approach (Section~\ref{sec:fast_lt}) for handling causal masks.
For both Performer and Polysketch, all attention heads share the same $\phi'$ within the same attention layer.

\paragraph{Hyper-parameters.} For FlashAttention, we try both block size $256$ and $512$\footnote{We find a speed-up increasing the default $128$ block size to $256$ and $512$ under our experimental setting. When increasing the block size to $1024$, FlashAttention ran out of memory under our empirical setup.}.
 For our fast lower triangular multiplication approach, we use $b=1024$ for both Polysketch and Performer.
We test both sketch sizes $r=32$ and $r=64$ for our Polysketch attention.
We use 2048 features for Performer\footnote{When using 4096 features, Performer ran out of memory in our experiments.}.

\paragraph{Pre-training metrics measurements (perplexities) over different context lengths.}
We train GPT-2 style small scale models equipped with different attention mechanisms on the Wiki-40B \citep{guo2020wiki} and PG-19 \citep{raecompressive2019} datasets with context length from 512 to 32k where each training batch contains 1M tokens.
For all kernel based attentions (Performer and Polysketch), we use 13 layers instead of 12.
More training details are mentioned in Appendix~\ref{sec:ppl_results}.
The perplexity results are shown in Figure~\ref{fig:context-length-vs-ppl} and training latencies are shown in Figure~\ref{fig:training_latency}.
Due to the space limit, we put all exact numbers in Appendix~\ref{sec:ppl_results}  including a detailed discussion.
We observe that in the setting of 32k context length, Polysketch (learned + local, r=32) achieves \textbf{2x} speed-up in comparison with FlashAttention of the fastest setup. 
As shown in Table~\ref{tab:ppl-pg19} and Table~\ref{tab:ppl-wiki} in Appendix~\ref{sec:ppl_results}, when we increase the sketch size $r$ from $32$ to $64$, we further reduce the perplexities.
In addition, as shown in Table~\ref{tab:training-latencies}, Polysketch (learned + local, r=64) still keeps $\sim10\%$ speed-up in comparison with FlashAttention of the fastest setup.
In addition, we observed that every kernel-based attention approach (Performer and Polysketch) with fast lower triangular matrix multiplication method almost keeps the same speed across different  context lengths given that we use same number of training tokens per step.
See more discussions in Appendix~\ref{sec:ppl_results}.

\paragraph{Downstream tasks of language models.} 
We train our models at different scales on the C4 dataset where each training batch contains 0.5M tokens.
The training details can be found in Appendix~\ref{sec:down_stream}.
In Table~\ref{tab:main-table}, we report the perplexity on the validation split of C4 dataset and 0-shot and 5-shot accuracies on a random sample of 500 examples of HellaSwag \cite{zellers2019hellaswag}, 500 examples of PIQA \cite{bisk2020piqa} and on the full Physics question answering dataset \cite{physics}. 
In addition to training models using 125k steps, we also train models with 30k steps to observe how the performance of attention mechanisms evolve with increasing number of total tokens trained on.
The results for 30k steps can be found in Appendix~\ref{sec:down_stream}.
As observed from Table~\ref{tab:main-table}, Polysketch attention has a comparable performance and sometimes outperforms softmax attention. 
In addition, the model quality improved with increasing model sizes. 
We leave more discussions in Appendix~\ref{sec:down_stream}.

\section{Acknowledgements}
\vspace{-0.1in}
We would like to thank Zeyuan Allen-Zhu, Krzysztof Choromanski, Insu Han, Yanping Huang, Weizhe Hua, Rajesh Jayaram, Zhipeng Jia, Amin Karbasi, Tamas Sarlos, and David P. Woodruff for helpful discussions and comments for improving the detailed implementations and presentations.
We would like to also thank many other contributors in the JAX/Flax community for suggestions on additional implementation details.

\bibliography{main}
\bibliographystyle{icml2024}

\newpage
\appendix
\onecolumn

\section{Conclusion and Future Work}

In this work, we empirically studied the performance of using high degree polynomial attention instead of softmax attention in training decoder-only models for language modeling tasks.
Our empirical study shows that the polynomial attention can achieve a similar model quality as the vanilla softmax attention when degree $p\geq 4$.
Then we developed an efficient approximate polynomial attention via polynomial sketching techniques which can be computed in linear time of  context length with provable approximation guarantees.
In addition, we presented a fast block based lower triangular matrix multiplication algorithm which can significantly boost the training time of any kernel based attention in the decoder based models.

There are several potential directions for future works. 
(1) Although we only empirically studied the performance of decoder-only models with polynomial attention for language modeling tasks, it is interesting to explore the potentials of encoder models with polynomial attention, and to understand whether it can be used in other fields such as vision.
(2) In this work, empirically we mainly focus on reducing the training latency. The benefits of linear transformers also transfer to inference as the KV cache sizes are independent of the context length. The exact inference improvements using linear transformers have to be explored more thoroughly.
(3) Polysketch attention is a kernel based method which can compute dense attention in linear time. It is interesting to see whether it can be combined with sparsification based efficient attention techniques such as HyperAttention proposed by~\cite{han2023hyperattention} recently.

\section{Discussion of the error bound of Theorem~\ref{thm:main_in_intro}}\label{sec:discussion_error_bound}

Let us look closely to the error bound stated in Theorem~\ref{thm:main_in_intro}.
Our error only has polynomial dependence in the $\ell_2$ norm bounds of $\{\bq_i\}$ and $\{\bk_j\}$.
In other word, to keep the same error, our sketching dimension $r$ only has polynomial dependence in the $\ell_2$ norm bounds of $\{\bq_i\}$ and $\{\bk_j\}$.
In contrast, to approximate the exponential kernel, the sketching dimension of Performer~\cite{choromanski2020rethinking} grows exponentially in the $\ell_2$ norm bounds of $\{\bq_i\}$ and $\{\bk_j\}$.

Suppose the $\ell_2$ norm of query and key vectors is bounded, i.e., $\max_{i}\max(\|\bq_i\|_2,\|\bk_i\|_2)\leq C$.
In the softmax attention, the ratio between two attention weights can be at most $\exp(\langle \bq_i, \bk_j \rangle) / \exp(\langle \bq_i, \bk_{j'}\rangle)\leq \exp(C^2)/\exp(-C^2)=\exp(2C^2)$ which is bounded.
In contrast, $\langle \bq_i, \bk_j \rangle^p / \langle \bq_i, \bk_{j'}\rangle^p$ can be arbitrarily large since $\langle \bq_i, \bk_{j'} \rangle$ can be close to $0$.
Therefore, in this bounded norm situation, polynomial attention is more capable for the operation of "taking the max".

Another difference between the approximation provided by Performer and ours is that Performer provides entry-wise approximation guarantee while we provide an approximation guarantee in average. 
Consider an example that all query and key vectors have $\ell_2$ norm at most $1$.
By  Markov inequality, we know $90\%$ of pairs $(i,j)\in [n]\times [n]$ satisfies $\left|\langle \phi'(\bq_i), \phi'(\bk_j) \rangle - \langle \bq_i, \bk_j\rangle^2\right|\leq \varepsilon'$, where $\varepsilon'=10\varepsilon$.
As long as both $\langle \bq_i, \bk_j\rangle^p, \langle \bq_i, \bk_{j'}\rangle^p \in (\varepsilon'/\varepsilon'' , 1]$ for some arbitrary $\varepsilon''>\varepsilon'$, $\langle \phi'(\bq_i),\phi'(\bk_j)\rangle/\langle \phi'(\bq_i),\phi'(\bk_{j'})\rangle$ is a $(1\pm O(\varepsilon''))$-approximation to $\langle \bq_i, \bk_j\rangle^p/ \langle \bq_i, \bk_{j'}\rangle^p$.

\section{Proof of Theorem~\ref{thm:tensoring-preserves-amm}}\label{sec:proof_of_tensoring}

We first note the following fact: If $\bS$ has $(\varepsilon, \delta, t)$-JL moment property, then for any two arbitrary vectors $\bx$ and $\by$, we have that $\|\la \T\bS\bx, \T\bS\by\ra - \la \bx, \by\ra\|_{L^t} \le \varepsilon\delta^{1/t}\opnorm{\bx}\opnorm{\by}$. For a proof see Lemma~9 from \cite{ahle2020oblivious}. 
\begin{proof}[Proof of Theorem~\ref{thm:tensoring-preserves-amm}]
Let $\bc_i$ denote the $i$-th row of $\bC$ and $\bd_j$ denote the $j$-th row of $\bD$. Then the $(i,j)$-th entry of the matrix $\bC^{\otimes 2}\T{(\bD^{\otimes 2})}$ is equal to $\la \bc_i, \bd_j\ra^2$. Similarly, the $(i,j)$-th coordinate of the matrix $(\bC\bS)^{\otimes 2}\T{((\bD\bS)^{\otimes 2})}$ is equal to $\la \T{\bS}\bc_i, \T{\bS}\bd_j\ra^2$ and therefore
\begin{align*}
    \frnorm{(\bC\bS)^{\otimes 2}\T{((\bD\bS)^{\otimes 2})} - \bC^{\otimes 2}\T{(\bD^{\otimes 2})}}^2 = \sum_{i,j}(\la \T{\bS}\bc_i, \T{\bS}\bd_j\ra^2 - \la \bc_i, \bd_j\ra^2)^2. 
\end{align*}
Recall that given an integer $t \ge 1$, for a random variable $\bX$, we define $\|\bX\|_{L^t}$ as $\E[|\bX|^t]^{1/t}$. Also note that $\|\bX\|_{L^t}$ is a norm over the random variables and in-particular satisfies the triangle inequality. Now,
\begin{align*}
    \|\frnorm{(\bC\bS)^{\otimes 2}\T{((\bD\bS)^{\otimes 2})} - \bC^{\otimes 2}\T{(\bD^{\otimes 2})}}\|_{L^t} &= \|\frnorm{(\bC\bS)^{\otimes 2}\T{((\bD\bS)^{\otimes 2})} - \bC^{\otimes 2}\T{(\bD^{\otimes 2})}}^2\|_{L^{t/2}}^{1/2}\\
    &=\|\sum_{i,j} (\la \T{\bS}\bc_i, \T{\bS}\bd_j\ra^2 - \la \bc_i, \bd_j\ra^2)^2\|_{L^{t/2}}^{1/2}\\
    &\le ({\sum_{i,j}\|(\la \T{\bS}\bc_i, \T{\bS}\bd_j\ra^2 - \la \bc_i, \bd_j\ra^2)^2\|_{L^{t/2}}})^{1/2}
\end{align*}
where we used the triangle inequality of $\|\cdot\|_{L^t}$ in the last inequality. Now consider a single term $\|(\la \T{\bS}\bc_i, \T{\bS}\bd_j\ra^2 - \la \bc_i, \bd_j\ra^2)^2\|_{L^{t/2}}$. First, we have
\begin{align*}
    &(\la\T{\bS}\bc_i, \T{\bS}\bd_j\ra^2 - \la \bc_i, \bd_j\ra^2)^2\\
    &= (\la\T{\bS}\bc_i, \T{\bS}\bd_j\ra + \la \bc_i,\bd_j\ra)^2(\la \T{\bS}\bc_i, \T{\bS}\bd_j\ra - \la \bc_i, \bd_j\ra)^2\\
    &= (\la\T{\bS}\bc_i, \T{\bS}\bd_j\ra - \la \bc_i,\bd_j\ra + 2 \la \bc_i, \bd_j\ra)^2(\la \T{\bS}\bc_i, \T{\bS}\bd_j\ra - \la \bc_i, \bd_j\ra)^2\\
    &\le (1+C)(\la \T{\bS}\bc_i, \T{\bS}\bd_j\ra - \la \bc_i, \bd_j\ra)^4 + 4(1+1/C)\la \bc_i, \bd_j\ra^2(\la \T{\bS}\bc_i, \T{\bS}\bd_j\ra - \la \bc_i, \bd_j\ra)^2
\end{align*}
with probability $1$ for any $C \ge 1$. Since both LHS and RHS are \emph{non-negative} random variables, we obtain that 
\begin{align*}
    &\|(\la \T{\bS}\bc_i, \T{\bS}\bd_j\ra^2 - \la \bc_i, \bd_j\ra^2)^2\|_{L^{t/2}}\\
    &\le (1+C)\|(\la \T{\bS}\bc_i, \T{\bS}\bd_j\ra - \la \bc_i, \bd_j\ra)^4\|_{L^{t/2}} + 4(1+1/C)\la \bc_i, \bd_j\ra^2\|(\la\T{\bS}\bc_i, \T{\bS}\bd_j\ra - \la \bc_i, \bd_j\ra)^2\|_{L^{t/2}}.
\end{align*}
Now,
\begin{align*}
    \|(\la \T{\bS}\bc_i, \T{\bS}\bd_j\ra -\la \bc_i, \bd_j\ra)^4\|_{L^{t/2}} &= \|\la \T{\bS}\bc_i, \T{\bS}\bd_j\ra - \la \bc_i, \bd_j\ra\|_{L^{2t}}^4\\
    &\le \varepsilon^4\delta^{2/t}\opnorm{\bc_i}^4\opnorm{\bd_j}^4
\end{align*}
assuming that $S$ has $(\varepsilon, \delta, 2t)$-JL moment property. We also have
\begin{align*}
    \|(\la \T{\bS}\bc_i, \T{\bS}\bd_j\ra -\la \bc_i, \bd_j\ra)^2\|_{L^{t/2}} &= \|\la \T{\bS}\bc_i, \T{\bS}\bd_j\ra - \la \bc_i, \bd_j\ra\|_{L^t}^2\\
    &\le \varepsilon^2\delta^{2/t}\opnorm{\bc_i}^2\opnorm{\bd_j}^2
\end{align*}
assuming that $\bS$ has $(\varepsilon, \delta, t)$-JL moment property. Overall, we get
\begin{align*}
    &\|(\la \T{\bS}\bc_i, \T{\bS}\bd_j\ra^2 - \la \bc_i, \bd_j\ra^2)^2\|_{L^{t/2}}\\
    &\le (1+C)\varepsilon^4\delta^{2/t}\opnorm{\bc_i}^4\opnorm{\bd_j}^4 + 4(1+1/C)\la \bc_i, \bd_j\ra^2\varepsilon^2\delta^{2/t}\opnorm{\bc_i}^2\opnorm{\bd_j}^2.
\end{align*}
Picking $C = 1/\varepsilon$ and assuming $\varepsilon \le 1/5$, we get that
\begin{align*}
    \|(\la \T{\bS}\bc_i\ra^2 - \la \bc_i, \bd_j\ra^2)^2\|_{L^{t/2}} \le 5\varepsilon^2\delta^{2/t}\opnorm{\bc_i}^4\opnorm{\bd_j}^4.
\end{align*}
Thus, we have
\begin{align*}
    \|\frnorm{(\bC\bS)^{\otimes 2}\T{((\bD\bS)^{\otimes 2})} - \bC^{\otimes 2}\T{(\bD^{\otimes 2})}}\|_{L^t} &\le \sqrt{5}\varepsilon\delta^{1/t}\sqrt{\sum_{i, j}\opnorm{\bc_i}^4\opnorm{\bd_j}^4}\\
    &\le \sqrt{5}\varepsilon\delta^{1/t}\frnorm{\bC^{\otimes 2}}\frnorm{\bD^{\otimes 2}}.
\end{align*}
By using Markov's inequality, we obtain that with probability $\ge 1 - \delta$,
\begin{align*}
    \frnorm{(\bC\bS)^{\otimes 2}\T{((\bD\bS)^{\otimes 2})} - \bC^{\otimes 2}\T{(\bD^{\otimes 2})}} \le \sqrt{5}\varepsilon\frnorm{\bC^{\otimes 2}}\frnorm{\bD^{\otimes 2}}.&\qedhere
\end{align*}
\end{proof}

\section{Replacing Random Projections with Learnable Transformations}\label{sec:network_structure_for_learnable_sketch}
Our learnable polynomial sketch algorithm is stated in Algorithm~\ref{alg:learnable_polysketch}.
It has a similar structure as our randomized polynomial sketch stated in Algorithm~\ref{alg:polysketch}.
The only differences are (1) we replace random projections $\bM_1\bG_1$ and $\bM_2\bG_2$ with $f_1(\bM_1)$ and $f_2(\bM_2)$ respectively.
(2) We apply a $\tanh(\cdot)$ trick to each entry of $\sqrt{1/r}\cdot [f_1(\bM_1)*f_2(\bM_2)]$ to make the output within a reasonable range and thus make the optimization process stable and converge.

\begin{algorithm}[ht]
\small
\caption{Learnable Polynomial Sketches}
\label{alg:learnable_polysketch}
\begin{algorithmic}
\FUNCTION{\textsc{LearnablePolysketchWithNegativity}$(\bA\in\mathbb{R}^{k\times m}, r, p)$}
\STATE // Analog of \textsc{PolysketchWithNegativity}.
\STATE If $p=1$, return $\bA$
\STATE $\bM_1$ = \textsc{LearnablePolysketchWithNegativity$(\bA, r, p/2)$}
\STATE $\bM_2$ = \textsc{LearnablePolysketchWithNegativity$(\bA, r, p/2)$}
\STATE Return $ \sqrt{r}\cdot\tanh\left(\sqrt{1/r}\cdot [f_1(\bM_1) * f_2(\bM_2)]\right)\in\mathbb{R}^{k\times r}$
\ENDFUNCTION
\FUNCTION{\textsc{LearnablePolysketchNonNegative}$(\bA\in\mathbb{R}^{k\times m}, r, p)$}
\STATE // Analog of \textsc{PolysketchNonNegative}. 
\STATE $\bM$ = \textsc{LearnablePolysketchWithNegativity$(\bA, r, p/2)$}
\STATE Return $\bM^{\otimes 2}\in\mathbb{R}^{k\times r^2}$.
\ENDFUNCTION
\end{algorithmic}
\end{algorithm}
Each $f_1(\cdot),f_2(\cdot)$ has the same dense network structure but different learnable parameters.
The network has output dimension $r$ and $3$ hidden layers with size $[8 r, r, 8 r]$.
We apply an activation function $\mathrm{gelu}(\cdot)$ after the first and the third hidden layer.
We apply an layer normalization before the input and the second hidden layer.
Therefore, each network only has roughly $8hr+24r^2$ or $32r^2$ number of parameters.
The entire learnable polynomial sketch only contains $p-2$ learnable networks.

Since all attention heads share the same learnable polynomial sketch within an attention layer, the number of increased learnable parameters is negeligible in comparison with the entire model.

Note that we did not take much time to optimize the network structure. 
It is likely that better network structures exist.
We leave the question of finding a better network structure as a future work.
\section{Perplexity Results on PG-19 and Wiki-40B}\label{sec:ppl_results}
We train GPT-2 small scale models on PG-19 and Wiki-40B datasets at various context lengths. We use the same training recipe that we described in Section~\ref{sec:exp} and train the models for 125k steps with a batch size of 1M tokens. For each of PG-19 and Wiki-40B datasets, we obtain a SentencePiece vocabulary of size 32k and train the models using the respective tokenizer.

We measure test perplexities of each of the models in Tables~\ref{tab:ppl-pg19} and \ref{tab:ppl-wiki}. We can see that the Polysketch attention model (learned + local) equipped with one additional layer beats the softmax attention models at all context lengths. 
We also note that Polysketch attention models, even without local attention, also achieve perplexities close to that of softmax models at all context lengths. These experiments show that our attention mechanism can scale to large context lengths without significant model quality loss.

The training latencies are shown in Table~\ref{tab:training-latencies}.
As we observed that all kernel based approaches equipped with our fast lower triangular multiplication approach are significantly faster than non-kernel based methods.
Notably, Polysketch (learned + local, r=64) achieves 1.1x speed up in comparison with the FlashAttention (block size 512) on 32k context length, and Polysketch (learned + local, r=32) achieves \textbf{2x} speed up in comparison with the FlashAttention (block size 512) on 32k context length.
Both Polysketch (learned + local, r=32) and Polysketch (learned + local, r=64) have lower perplexity than the softmax attention.
Polysketch (learned + local, r=64) has the lowest test perplexity.

\begin{table*}
\centering
\begin{tabular}{l c c c c c c c}\toprule
& 512 & 1k & 2k & 4k & 8k & 16k & 32k\\\midrule
\textbf{Non-kernel based methods, 12 layers}\\
  Softmax (using FlashAttention)   & {13.57} & {12.75}  & {12.23}  & {11.88}  & {11.65}  & {11.57}  & {11.55}   \\
  Polynomial (deg=2)     & {13.84} & {13.10} & {12.75} & {12.61} & {12.72} & OOM & OOM\\
  Polynomial (deg=4)     & {13.58} & {12.76} & {12.26} & {12.00} & {11.85} & OOM & OOM\\
  Polynomial (deg=8)     & {13.56} & {12.71} & {12.16} & {11.86} & {11.64} & OOM & OOM\\\midrule
  \textbf{Kernel based methods, 13 layers}\\
  Polysketch (random, r = 32) &   14.31   &   13.74   &   13.40   &   13.26   &   13.41   &   13.79   &   14.75\\
  Polysketch (random, r = 64) & {14.00} & {13.35} & {13.03} & {12.84} & {12.92} & {13.18} & {13.66}\\
  Polysketch (learned, r = 32)  &   13.49   &   12.74   &   12.34   &   12.16   &   12.21   &   12.40   &   12.79\\
  Polysketch (learned, r = 64) & \tbold{13.33} & {12.60} & {12.10} & {11.90} & {11.86} & {11.94} & {12.19}\\
  Polysketch (random + local, r = 32)   &   13.37   &   \textbf{12.58}   &   12.23   &   12.01   &   11.95   &   11.90   &   12.16\\
  Polysketch (random + local, r = 64) & {13.37} & \tbold{12.58} & {12.24} & {11.98} & {11.93} & {11.96} & {11.91}\\
  Polysketch (learned + local, r = 32) & {13.37} & \tbold{12.58} & {12.09} & {11.75} & {11.55} & {11.46} & {11.47}\\
  Polysketch (learned + local, r = 64) & {13.37} & \tbold{12.58} & \tbold{12.03} & \tbold{11.69} & \tbold{11.44} & \tbold{11.38} & \tbold{11.34}\\
  Performer (2048 features) & {14.30} & {13.68} & {13.56} & {13.50} & {13.49} & {13.73} & {14.17}\\ \bottomrule
\end{tabular}
\caption{Perplexities on the test split of PG19 when the models are trained on PG19 dataset}
\label{tab:ppl-pg19}
\end{table*}

\begin{table*}
\centering
\begin{tabular}{l c c c c c c c}\toprule
& 512 & 1k & 2k & 4k & 8k & 16k & 32k\\\midrule
\textbf{Non-kernel based methods, 12 layers}\\
  Softmax (using FlashAttention)   &   15.82 & 15.04 & 14.61 & 14.40 & 14.35 & 14.34 & 14.35 \\
  Polynomial (p=2)     & 16.24 & 15.58 & 15.38 & 15.41 & 15.60 & OOM & OOM \\
  Polynomial (p=4)     & 15.85 & 15.11 & 14.75 & 14.59 & 14.59 & OOM & OOM\\
  Polynomial (p=8)     & 15.81 & 15.00 & 14.56 & 14.36 & 14.32 & OOM & OOM\\ \midrule
  \textbf{Kernel based methods, 13 layers}\\
  Polysketch (random, r = 32) & 16.84 & 16.35 & 16.20 & 16.28 & 16.52 & 17.05 & 17.84\\
  Polysketch (random, r = 64) & 16.32 & 15.73 & 15.72 & 15.88 & 16.01 & 16.44 & 17.45\\
  Polysketch (learned, r = 32) & 15.84 & 15.20 & 14.95 & 14.91 & 15.06 & 15.52 & 15.93\\
  Polysketch (learned, r = 64) & 15.65 & 14.96 & 14.62 & 14.58 & 14.70 & 14.95 & 15.35\\
  Polysketch (random + local, r = 32) & \tbold{15.63} & \tbold{14.86} & 14.60 & 14.50 & 14.52 & 14.65 & 14.68\\
  Polysketch (random + local, r = 64) & \tbold{15.63} & \tbold{14.86} & 14.58 & 14.52 & 14.43 & 14.62 & 14.54 \\
  Polysketch (learned + local, r = 32) & \tbold{15.63} & \tbold{14.86} & 14.46 & 14.28 & 14.24 & \tbold{14.23} & 14.32\\
  Polysketch (learned + local, r = 64) & \tbold{15.63} & \tbold{14.86} & \tbold{14.43} & \tbold{14.26} & \tbold{14.18} & {14.24} & \tbold{14.29}\\
  Performer (2048 features) & 16.75 & 16.18 & 16.14 & 16.37 & 16.64 & 17.16 & 18.40\\\bottomrule
\end{tabular}
\caption{Perplexities on the test split of Wiki-40B when the models are trained on Wiki-40B dataset}
\label{tab:ppl-wiki}
\end{table*}

\begin{table*}
    \centering
    \begin{tabular}{l c c c c c c c c c}\toprule
                            &  512     & 1k   &     2k  &    4k     &   8k   &   16k     &   32k \\\midrule
        Softmax             & \textbf{6.00}       & {4.95}   & {3.65} & 1.95 & 1.19 & OOM & OOM\\\midrule
        \begin{tabular}{@{}l@{}}FlashAttention\\ (Block size 256 x 256)\end{tabular}                                                       & 4.78      &  4.09  & 3.25 & {2.31} & 1.47 & 0.85 & 0.45 \\\midrule
        \begin{tabular}{@{}l@{}}FlashAttention\\ (Block size 512 x 512)\end{tabular}                                                       & 5.46      &  \textbf{5.0}  & \textbf{4.4} & \textbf{3.65} & \textbf{2.7} & 1.8 & 1.03 \\\midrule
        \begin{tabular}{@{}l@{}}Polynomial \\ (p=2, 4, 8)\end{tabular}                                              &  5.74     &  4.74  & 3.55 & 2.16 & 1.20 & OOM & OOM\\\midrule
        \begin{tabular}{@{}l@{}}Polysketch \\ (random, 13 layers, r = 32)\end{tabular}                    &    5.25    &   4.31    &   3.66    &   3.54    &   3.39    &   3.25    &   2.56

 \\\midrule
        \begin{tabular}{@{}l@{}}Polysketch \\ (random, 13 layers, r = 64)\end{tabular}                    &    5.06   &    4.20 & 2.50 & 2.23 & 2.06 & 1.95 & 1.55 \\\midrule
        \begin{tabular}{@{}l@{}}Polysketch \\ (learned, 13 layers, r = 32)\end{tabular}                    &   3.16
&   2.82    &   2.45    &   2.40    &   2.33    &   2.26    &   2.00\\\midrule
        \begin{tabular}{@{}l@{}}Polysketch \\ (learned, 13 layers, r = 64)\end{tabular}                    &   2.17    & 1.97   & 1.37  & 1.35 & 1.33 & 1.29 & 1.13 \\\midrule
        \begin{tabular}{@{}l@{}}Polysketch \\ (random + local, 13 layers, r = 32)\end{tabular}                  & 5.35  &   4.40    &   3.24    &   3.2 &   3.14    &   3.11    &   2.71\\\midrule
        \begin{tabular}{@{}l@{}}Polysketch \\ (random + local, 13 layers, r = 64)\end{tabular}                  & 5.35  &    4.40   &    2.09 & 2.00 & 1.91 & 1.82  & 1.60  \\\midrule
        \begin{tabular}{@{}l@{}}Polysketch \\ (learned + local, 13 layers, r = 32)\end{tabular}                    &  5.35     & 4.40   & 2.27 & 2.27 & {2.20} & \textbf{2.19} & \textbf{1.98}\\\midrule
        \begin{tabular}{@{}l@{}}Polysketch \\ (learned + local, 13 layers, r = 64)\end{tabular}                    &  5.35     & 4.40    & 1.31 & 1.31 & 1.30 & 1.26 & 1.12\\\midrule
        \begin{tabular}{@{}l@{}}Performer \\ (2k features\\+ Fast lower triangular multiplications)\end{tabular}                                         & 2.21    &    1.58       & 1.39 & 1.36 & 1.33 & 1.32 & 1.18\\\midrule
        \begin{tabular}{@{}l@{}}Performer \\ (256 features (default)\\without Fast lower triangular multiplications)\end{tabular}                                             &  0.44     & 0.40   & 0.36  &  0.29 & 0.21 & 0.14  & 0.08 \\  
        \bottomrule
    \end{tabular}
    \caption{\textbf{Training steps/sec} of different attention mechanisms at various context lengths (\textbf{higher is faster}). For context lengths 512 and 1k, we compute the full attention matrix in Polysketch and Performer attention \textbf{without} using the linearization technique. These models are all GPT-2-like small scale models. Each batch contains 1M tokens in total.}
\label{tab:training-latencies}
\end{table*}

\subsection{Training Latency Comparison}

The main advantage of linear transformers is that their training latency remains the same across different context lengths given that we use the same ``batch size'' (tokens per training step) for all the context lengths. To show that it is the case, we report the training latencies (in terms of steps/sec) of our models and other attention mechanisms in Table~\ref{tab:training-latencies}. Using the same batch size across different context lengths, we note that the steps/sec of linear transformers such as Polysketch and Performer remain almost constant whereas the steps/sec of quadratic-time transformers decreases with increasing context lengths. The results show that, depending on the model structure, models using our Polysketch attention mechanism are \textbf{significantly faster to train} than models using a quadratic attention mechanism such as softmax (implemented via FlashAttention) at \textbf{long context lengths}.

\section{Experiments with Synthetic Tasks}\label{apx:synthetic}
For both synthetic experiments, We train a small 2-layer transformer.
Each attention layer contains 8 attention heads where each has head size 16.
For Polysketch attention, we choose r=32 and the block size $b=1024$ for fast lower triangular multiplication (Section~\ref{sec:fast_lt}).

\subsection{Selective Copying}
Recently, \citet{gu2023mamba} have used selective copying task as a yard stick for measuring content aware reasoning capabilities and the memorization abilities of the models. In this task, the model is required to memorize colored blocks that appear in the context and the model needs to output the colored blocks in the same order at the end. See \cite{gu2023mamba} for a more detailed description of this task. 

We generate 64k random examples used for training. 
Each batch has 64 examples.
We train for 400k steps in total without otherwise specified.
Using a similar training recipe as in their paper, we train small two layer models using different attention mechanisms to solve these tasks at context lengths 4k, 16k, and 32k. We report our results in Table~\ref{tab:selective-copying}. We see that polynomial and Polysketch attention manage to learn to solve the selective copying task though the accuracy of Polysketch is a bit worse at 16k context length with the same training recipe as other models. We found that with a different learning rate schedule, Polysketch attention also manages to solve the selective copying task at a context length 16k with an accuracy of 99.44\% thus showing that there may not be any loss in matching the reasoning capabilities of the softmax attention mechanism. This suggests that Polysketch attention may require different learning rate schedules to obtain the optimal performance as compared to softmax transformers.

We also find that at a context length of 32k, Polysketch attention learns to solve 95.29\% of test examples after 800k steps of training. In all our experiments, we observe sudden spike in the accuracies of the models indicating the point where the model learns to solve the task, see e.g., Figure~\ref{fig:learn_curve}.

\begin{table}[ht]
    \centering
    \begin{tabular}{l c c c}\toprule
                 & 4k & 16k  & 32k  \\ \midrule
         Softmax &  99.73 \%& 98.17\% & 0\%\\ 
         Polynomial (degree=4)& 99.90\% & 97.97\% & 0\%\\
         Polynomial (degree=8)& 99.90\% & 97.65\% & 0\%\\
         Polysketch (learned + local) & 99.16\% & 92.75\% & 0\%\\ 
         \begin{tabular}{@{}l@{}}Polysketch (learned + local) \\ (different learning rate schedule)\end{tabular} & - & 99.44\% & 87.16\%\\ \bottomrule
    \end{tabular}
    \caption{\% of 4096 examples on which the models succeeded to perfectly output the colored blocks in the context in the same order.}
    \label{tab:selective-copying}
\end{table}

\begin{figure*}[t]
\centering
\includegraphics[width=0.8\linewidth]{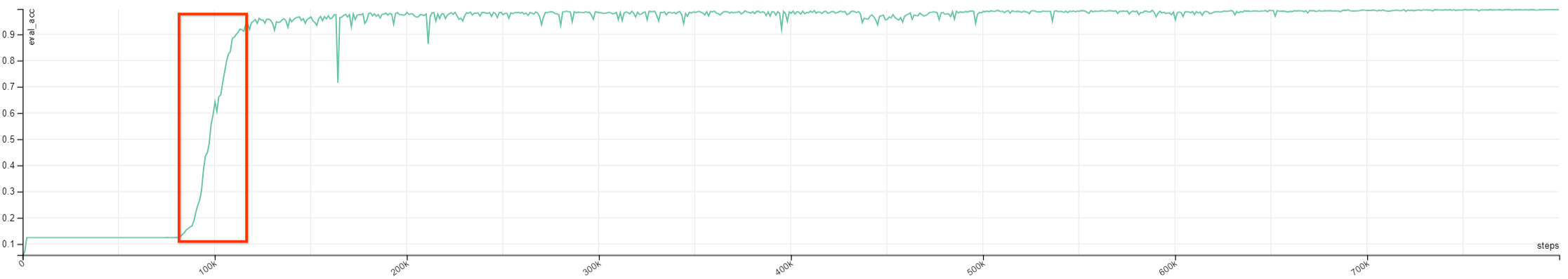}
\caption{The test accuracy during training Polysketch (learned + local) on the selective copying task of 32k context length.
The x-axis is the number of steps trained and the y-axis is the test accuracy.
We observe that the model suddenly learns the pattern at some point during the training.}\label{fig:learn_curve}
\end{figure*}
\subsection{Induction heads}
\citet{olsson2022context} have proposed induction heads task as a way to identify and explain the in-context learning capabilities of language models. This task requires the model to output the token that appears immediately after a special token that appears in the context exactly once at an arbitrary position. 

We generate 64k random examples used for training. 
Each batch contains 64 examples.
Each position is a random token from a vocabulary of size 16.
We replace a random position except the last 3 tokens with a special token.
We replace the second to the last token with a special token.
We replace the last token with the token appeared directly after the first special token.

The total number of training steps is 400k.
We consider context lengths 128 and 256. We observe that all the models (softmax, degree 4 polynomial, degree 8 polynomial and Polysketch with sketch size 16 and 32) are able to solve (accuracy > 99.95\%) the task at a context length of 128 and all of them fail to solve (accuracy around 1/16, i.e., random guessing) the task at a context length 256 under the same optimization configuration.

\begin{table*}[t!]
\centering
\begin{tabular}{l c c c c c c c}\toprule
&  C4& \multicolumn{2}{c}{HellaSwag} & \multicolumn{2}{c}{PIQA} & \multicolumn{2}{c}{Physics}\\
\cmidrule(lr){2-2}\cmidrule(lr){3-4}\cmidrule(lr){5-6}\cmidrule(lr){7-8}
& Perplexity $\downarrow$    & 0-shot $\uparrow$ & 5-shot $\uparrow$ & 0-shot $\uparrow$ & 5-shot $\uparrow$ & 0-shot $\uparrow$ & 5-shot $\uparrow$\\\midrule
\multicolumn{6}{l}{\textbf{GPT-2 Small style, 100M-scale, 12 layers default, Context Length 8192, 30k training steps}}\\
Softmax                          &   20.11 &   28.0    &   \textbf{28.8}    &   61.8    &   \textbf{62.8}    &   20.5    &   30.1           \\
Polynomial (degree 4)            &   20.66  &   \underline{28.4}    &   28.6    &   60.4    &   61.0    &   \underline{24.4}    &   \underline{30.1}                \\
Polynomial (degree 8)            &   20.05  &   27.6    &   27.4    &   59.8    &   60.0    &   \underline{20.5}    &   23.1\\
Polysketch (learned, r = 64)             &   21.16  &   27.8    &   28.0    &   60.4    &   61.2    &   \underline{29.6}    &   \underline{33.6}\\
Polysketch (learned, 13 layers, r = 64)  &   20.93  &   \underline{28.6}    &   28.2    &   \textbf{62.0}    &   61.8    &  \underline{30.1}     & \underline{\textbf{35.8}} \\
Polysketch (learned + local, r = 64)     &   20.30  &   \underline{28.4}    &   27.8    &   61.6    & \underline{\textbf{62.8}}    &  \underline{27.9}    &   \underline{34.9}           \\
Polysketch (learned + local, 13 layers, r = 64)&   \underline{\tbold{19.91}}   &  \underline{28.4}    &   27.6    &   61.2    &   61.0    &   \underline{31.0}    &   \underline{30.5}\\
Polysketch (learned, r = 32)                &   22.31   &   27.0    &   26.6    &   59.8    &   61.8    & \underline{\textbf{35.3}}    &   \underline{34.4}                \\
Polysketch (learned, 13 layers, r = 32)     &   22.15   &   \underline{\tbold{28.8}}    &   28.2    &   59.8    &   60.0    &   \underline{32.3}    &   \underline{31.0}\\
Polysketch (learned + local, r = 32)        &   20.28   &   \underline{28.4}    &   28.4    &   59.6    &   61.4    &   \underline{29.6}    &   \underline{31.4}\\
Polysketch (learned + local, 13 layers, r = 32) & \underline{19.94} &   \underline{28.6}    &   28.2    &   60.4    &   61.2    &   \underline{29.2}    &   \underline{30.5}\\\midrule
\multicolumn{6}{l}{\textbf{GPT-2 Medium style, 300M-scale, 24 layers default, Context Length 8192, 30k training steps}}\\
Softmax                          &  15.97   &   \textbf{32.0}    &   31.6    &   61.8    &   63.4    &   25.3    &   29.2          \\
Polynomial (degree 4)            &  16.46   &   29.2    &   29.8    &   \underline{63.2}    &   \underline{64.4}    &   \underline{30.5}    &   \underline{33.6}\\
Polynomial (degree 8)            &  16.12   &   31.4    &   31.4    &   \underline{64.2}    &   \underline{64.0}    &   \underline{27.0}    &   28.3\\
Polysketch (learned)             &  17.18   &   29.2    &   30.0    &   \underline{62.8}    &   \underline{64.0}    &   \underline{27.0}    &   \underline{31.0}\\
Polysketch (learned, 26 layers, r = 64)  &  17.06   &   29.8    &   31.0    &   \underline{64.8}    &   \underline{64.6}    &   \underline{28.3}    &   \underline{30.1}\\
Polysketch (learned + local, r = 64)     & 16.14   &   31.6    & \underline{32.6}    &  \underline{64.6}    &   \underline{64.4}    &   \underline{27.9}    &   \underline{31.4}\\
Polysketch (learned + local, 26 layers, r = 64)&  \underline{\textbf{15.95}}   &   31.8    &   31.4    &   \underline{63.8}    &  \underline{\textbf{66.0}}     &   20.0    &   27.9\\
Polysketch (learned, r = 32)                &   17.81   &   30.4    &   30.4    &   60.2    &   61.6    &  \underline{\textbf{34.4 }}    &   \underline{\textbf{34.0}}\\   
Polysketch (learned, 26 layers, r = 32)     &   17.47   &   29.2    &   30.2    &   61.6    &   62.0    &   \underline{28.3}    &   \underline{31.4}\\
Polysketch (learned + local, r = 32)        &   16.20   &   31.6    &  \underline{\textbf{33.0}}     &     \underline{\textbf{65.8}}      &  \underline{65.8} &   \underline{31.8}    &   28.8\\
Polysketch (learned + local, 26 layers, r = 32) &   16.02   &  31.8    &   \underline{31.8}    &   \underline{64.4}    &   \underline{65.0}    &   \underline{27.5}    &   \underline{31.4}\\\bottomrule
\end{tabular}
\caption{We compare the accuracies(\%, higher the better) of different models (all trained for 30k training steps) on three different Q/A tasks. HellaSwag and Physics tasks have 4 choices and PIQA task has 2 choices. We also report the perplexities (lower the better) on the validation split of C4 dataset. 
\textbf{Bolding} indicates the best model in the task, \underline{underlining} indicates beating softmax attention.}
    \label{tab:main-table-2}
\end{table*}
\section{Experiments on Downstream Tasks}\label{sec:down_stream}
\paragraph{Additional training details.} 
We train all models from scratch on the C4 dataset using a SentencePiece tokenizer trained on C4 with a vocabulary size of 32,000. We use a batch size of 0.5M tokens per training step and a peak learning rate of 3e-4. We use a linear learning rate schedule to warmup the learning rate for the first 10\% of iterations and then again use a linear learning rate schedule to decay the learning rate. We use Adam optimizer with weight decay parameters $(\beta_1 = 0.95, \beta_2 = 0.98)$ in all our experiments.

The training latencies of small scale models on 8k context length are:
\begin{enumerate}
    \item Softmax (without FlashAttention): 2.40 step/sec.
    \item Polynomial (p=4,8): 2.65 step/sec.
    \item Polysketch (learned, r=64, 12 layers): 2.71 step/sec.
    \item Polysketch (learned, r=32, 12 layers): 4.68 step/sec.
    \item Polysketch (learned, r=64, 13 layers): 2.49 step/sec.
    \item Polysketch (learned, r=32, 13 layers): 4.34 step/sec.
    \item Polysketch (learned + local, r=64, 12 layers): 2.70 step/sec.
    \item Polysketch (learned + local, r=32, 12 layers): 4.42 step/sec.
    \item Polysketch (learned + local, r=64, 13 layers): 2.48 step/sec.
    \item Polysketch (learned + local, r=32, 13 layers): 4.12 step/sec.
\end{enumerate}
Note that the difference between above training latencies and those presented in Table~\ref{tab:training-latencies} is due to the different number of tokens per batch (0.5M vs 1M).

The training latencies of medium scale models on 8k context length are reported as follows:
\begin{enumerate}
    \item Softmax (without FlashAttention): 0.87 step/sec.
    \item Polynomial (p=4,8): 0.89 step/sec.
    \item Polysketch (learned, r=64, 24 layers): 0.99 step/sec.
    \item Polysketch (learned, r=32, 24 layers): 1.62 step/sec.
    \item Polysketch (learned, r=64, 26 layers): 0.92 step/sec.
    \item Polysketch (learned, r=32, 26 layers): 1.52 step/sec.
    \item Polysketch (learned + local, r=64, 24 layers): 0.98 step/sec.
    \item Polysketch (learned + local, r=32, 24 layers): 1.59 step/sec.
    \item Polysketch (learned + local, r=64, 26 layers): 0.91 step/sec.
    \item Polysketch (learned + local, r=32, 26 layers): 1.46 step/sec.
\end{enumerate}
For large scale models, since the context length is only 2k (recall that as we mentioned earlier, non-kernel based methods are either too slow or facing OOM issues for longer context length for the large-scale), the running time of kernel based methods do not take advantage from linearization. 
The purpose is to compare the model quality only.
So we omit the training latencies of large scale models here.

\paragraph{Additional results.}
The results similar to Table~\ref{tab:main-table} but for the models trained on only 30k steps is shown in Table~\ref{tab:main-table-2}.

\subsection{Scaling with Model Sizes}
From the results in Table~\ref{tab:main-table} and Table~\ref{tab:main-table-2}, we have the following main observations: (i) Polysketch attention (learned + local) \textbf{closely matches} the performance and sometimes outperforms models trained with softmax attention.
(ii) Models trained with Polysketch attention improve with increasing model sizes showing promise to be a replacement for softmax even in the largest models to \textbf{achieve lower training latencies} without significant \textbf{performance issues}.
(iii) Strong performance of learned Polysketch attention, \textbf{without relying on local attention}, shows the capability of our proposed attention mechanism and that the results that we obtain with learned+local Polysketch attention are \textbf{not just due to using exact polynomial attention within the blocks}.

\section{Model Sizes for Each Attention Mechanism}\label{sec:model_sizes}
In this section, we include the sizes of all models that we trained.

\subsection{Small scale models}
The default configuration has 12 layers, 12 attention heads per layer, each attention head has head size 64.
The number of parameters of each model is stated as the following.

12 layer models:
\begin{enumerate}
\item Softmax, Polynomial (degree = 2, 4 \& 8): 110M
\item Polysketch (learned, sketch size = 64, 12 layers), Polysketch (learned + local, sketch size = 64, 12 layers): 113M
\item Polysketch (learned, sketch size = 32, 12 layers), Polysketch (learned + local, sketch size = 32, 12 layers): 111M
\end{enumerate}

13 layer models:
\begin{enumerate}
\item Polysketch (learned, sketch size = 64, 13 layers), Polysketch (learned + local, sketch size = 64, 13 layers): 120M
\item Polysketch (learned, sketch size = 32, 13 layers), Polysketch (learned + local, sketch size = 32, 13 layers): 118M
\item Polysketch (random, sketch size = 32, 64, 13 layers), Polysketch (random + local, sketch size = 32, 64, 13 layers), Performer (2k features): 117M
\end{enumerate}

\subsection{Medium scale models}
The default configuration has 24 layers, 16 attention heads per layer, each attention head has head size 64.
The number of parameters of each model is stated as the following.

24 layer models:
\begin{enumerate}
\item Softmax, Polynomial (degree = 4, 8): 337M
\item Polysketch (learned, sketch size = 64, 24 layers), Polysketch (learned + local, sketch size = 64, 24 layers): 341M
\item Polysketch (learned, sketch size = 32, 24 layers), Polysketch (learned + local, sketch size = 32, 24 layers): 337M
\end{enumerate}

26 layer models:
\begin{enumerate}
\item Polysketch (learned, sketch size = 64, 26 layers), Polysketch (learned + local, sketch size = 64, 26 layers): 367M
\item Polysketch (learned, sketch size = 32, 26 layers), Polysketch (learned + local, sketch size = 32, 26 layers): 362M
\end{enumerate}

\subsection{Large scale models}
The default configuration has 36 layers, 20 attention heads per layer, each attention head has head size 64.
The number of parameters of each model is stated as the following.

36 layer models:
\begin{enumerate}
\item Softmax, Polynomial (degree = 4, 8): 748M
\end{enumerate}

39 layer models:
\begin{enumerate}
    \item Polysketch (learned, sketch size = 64, 39 layers), Polysketch (learned + local, sketch size = 64, 39 layers): 817M
    \item Polysketch (learned, sketch size = 32, 39 layers), Polysketch (learned + local, sketch size = 32, 39 layers): 811M
\end{enumerate}

\section{Reciepe of Transformer++}\label{sec:transformer_pp}
We add sinusoidal position embeddings~\cite{vaswani2017attention} to the input embeddings and use Rotary Position Embeddings (RoPE)~\cite{su2021roformer} at all attention heads.
 We use Gated Linear Units \citep{dauphin2017language, shazeer2020glu} with an expansion factor of 4 as the FeedForward layer in the network.
We use GELU as the non-linearity. 
All models are trained using Adam optimizer with weight decay and a peak learning rate of 7e-4. 


\end{document}